\documentclass[twoside]{article}

\usepackage[accepted]{aistats2026}
\usepackage[letterpaper, margin=1in]{geometry}

\usepackage{amsfonts, amssymb, amsthm}
\usepackage{amsmath}

\usepackage[round,authoryear]{natbib}
\makeatletter
\renewcommand\@biblabel[1]{}
\makeatother
\usepackage{authblk}
\usepackage[utf8]{inputenc} 
\usepackage[T1]{fontenc}    
\usepackage{hyperref}       
\usepackage{url}            
\usepackage{booktabs}       
\usepackage{amsfonts}       
\usepackage{nicefrac}       
\usepackage{microtype}      
\usepackage{cancel}

\usepackage{caption}
\usepackage{subcaption}
\usepackage{dsfont}
\usepackage{mathrsfs}
\usepackage{verbatim}
\usepackage{graphicx}
\usepackage{blindtext}
\usepackage{color, colortbl}
\usepackage[table]{xcolor}
\usepackage{algorithm}
\usepackage{enumitem}
\usepackage[noend]{algpseudocode}
\usepackage{bbm}
\usepackage{cleveref}
\usepackage{mathtools}
\usepackage{multicol}

\usepackage{hyperref}
\hypersetup{
    colorlinks,
    citecolor=OliveGreen,
    filecolor=black,
    linkcolor=black,
    urlcolor=black
}
\usepackage[dvipsnames]{xcolor}

\DeclarePairedDelimiter\ceil{\lceil}{\rceil}

\DeclareMathOperator*{\argmin}{argmin}

\newtheorem{thm}{Theorem}[section]
\newtheorem{prp}{Proposition}[section]
\newtheorem{lem}{Lemma}[section]
\newtheorem{cor}{Corollary}[section]
\newtheorem{defi}{Definition}[section]
\newtheorem{remark}{Remark}[section]

\definecolor{color1}{rgb}{0.0, 0.51, 0.58}

\begin{document}

%

%

\twocolumn[
\aistatstitle{Empirical PAC-Bayes Bounds for Markov Chains}
\aistatsauthor{Vahe Karagulyan \And Pierre Alquier}
\aistatsaddress{ESSEC Business School, Cergy \And ESSEC Business School, Singapore}]

\begin{abstract} The core of generalization theory was developed for independent observations. Some PAC and PAC-Bayes bounds are available for data that exhibit a temporal dependence. However, there are constants in these bounds that depend on properties of the data-generating process: mixing coefficients, mixing time, spectral gap... Such constants are unknown in practice. In this paper, we prove a new PAC-Bayes bound for Markov chains. This bound depends on a quantity called the \textit{pseudo-spectral gap}, $\gamma_{ps}$. The main novelty is that we can provide an empirical bound on $\gamma_{ps}$ when the state space is finite.
Thus, we obtain the first fully empirical PAC-Bayes bound for Markov chains. This extends beyond the finite case, although this requires additional assumptions. On simulated experiments, the empirical version of the bound is essentially as tight as the one that depends on $\gamma_{ps}$.

\end{abstract}

\section{INTRODUCTION}

The PAC-Bayes theory is a flexible framework to derive generalization guarantees for learning algorithms. Since the seminal paper by~\citet{mc1998}, it has found applications across a wide range of domains: from classification and regression to deep learning and variational inference. It has become a popular method in the machine learning community. One of the most successful applications of PAC-Bayes bounds was to obtain non-vacuous generalization bounds for deep neural networks~\citet{DBLP:journals/corr/DziugaiteR17}. We refer the reader to~\citet{Alquier_2024,hellstrom2025generalization} for a recent overview on PAC-Bayes bounds. \\
\\
The original paper by~\citet{mc1998} was written in the context of i.i.d. observations.
Since then, PAC-Bayes bounds have been extended to handle more challenging settings, such as data with temporal or spatial dependencies. We mention \citet{alq_wint_fast,e23030313,pmlr-v130-haussmann21a,41f14ea790bf4ab0a2e153e453036d2b} among others.  However, all these bounds involve constants characterizing the dependence in the data generating process: for example, the bound of~\citet{alq_wint_fast} depends on
weak-dependence coefficients, while the one of~\citet{e23030313} depends on $\alpha$-mixing coefficients, etc. The strategy adopted by the authors was to assume \textit{a priori} upper bounds on these constants. However, if this assumption is not correct, the PAC-Bayes bound is no longer valid. It would be much more satisfactory to estimate these constants and make the bounds fully empirical.
\\

In this paper, we provide fully empirical PAC-Bayes bounds for a fundamental class of processes: Markov chains.
We prove a PAC-Bayes bound that holds when the data is the trajectory of a Markov chain. This bound depends on a parameter $\gamma_{ps}$ called the pseudo-spectral gap of the transition operator of the chain~\citep{paulin2015}. We then provide an empirical bound on $\gamma_{ps}$ when the state space is finite, using tools from~\cite{wolfer2024improved}. Putting everything together leads to fully-empirical PAC-Bayes bounds. Empirical bounds on $\gamma_{ps}$ can be provided beyond the finite case, but this will in general require more assumptions on the chain. For example, we provide such a bound when the data is an autoregressive process.
\\

\subsection{Related Works}

\textbf{PAC-Bayes bounds:} while the original PAC-Bayes bounds were proven by \citet{mc1998,mc1999}, our proof will follow the alternative approach by~\citet{catoni2003}. This approach is summarized in~\cite{Alquier_2024}. Important references include \cite{Catoni2004APA,cat2007,hellstrom2025generalization}.
\\

\textbf{Empirical PAC-Bayes bounds:} the bounds in the above references are empirical, that is, they depend on the data but not on the unknown data-generating process. There were many attempts to make these bounds tighter~\cite{see2002,mau2004,kuzborskij2024better} or to extend them in various directions~\cite{seldin_laviolette,alquier2018simpler,rodriguez2024more}. A natural way to make the bounds tighter relies on the application of Bernstein's inequality: this gives bounds that are in principle tighter, but that depend on the variance of the loss function under the data-generating process~\citep{catoni2003}. This quantity is unknown in practice. In order to make the PAC-Bayes-Bernstein bound practical, it is necessary to provide an empirical upper bound on the variance term. The first occurrence of such an "empirical-PAC-Bayes-Bernstein" is due to~\citet{pmlr-v26-seldin12a} and these bounds were refined in later works~\citep{NIPS2013_a97da629,NEURIPS2019_3dea6b59,NEURIPS2021_69386f6b,NEURIPS2022_49ffa271,pmlr-v195-jang23a}.
\\

\textbf{PAC-Bayes bounds for dependent observations:}
PAC-Bayes bounds for Markov chains were proven by \citet{fard2011pac} in the setting of Markov decision processes, the bounds depend on ergodic coefficients that are unknown in practice. Also for Markov chains, \citet{e23030313} proved bounds that depend  explicitly on $\alpha$-mixing coefficients (also unknown in practice). Some bounds in~\cite{alquier2018simpler} which also depend on $\alpha$-mixing coefficients hold for a more general class of stochastic processes.
We also mention~\citet{pmlr-v130-haussmann21a} for continuous dynamical systems and \citet{41f14ea790bf4ab0a2e153e453036d2b} for linear time-invariant (LTI) systems.

There are other type of generalization bounds beyond PAC-Bayes, such as stability bounds and bounds based on the Rademacher complexity. Such results were also extended from the i.i.d. setting to time series under various assumptions:~\citet{yu1994rates,gamarnik1999extension,meir2000nonparametric,steinwart2009learning,mohri10a,modha2002memory,steinwart2009learning,shalizi2013predictive,kuznetsov2015learning,kuznetsov2017generalization,mcdonald2017nonparametric,kuznetsov2020discrepancy,abeles2024generalization}. These bounds also depend on mixing or weak-dependence coefficients. In the case of Markov chains, some bounds also depend on the mixing time of the chain, on its spectral gap or on its pseudo-spectral gap~\citet{garnier2023hold,alquier2025minimaxoptimalitydeepneural}.

\textbf{Estimation of the mixing coefficients:} there were recently attempts to estimate the mixing coefficients~\citep{mcdonald2015estimating,khaleghi2023inferring}. Although it is possible to estimate the mixing coefficients $\alpha$ and $\beta$, the results of~\cite{khaleghi2023inferring} do not provide confidence intervals on this estimation, and thus, cannot be used to derive empirical PAC-Bayes bounds. Some recent progress has been made in estimating the $\beta$-mixing coefficients of Markov chains~\citep{grunewalder2024estimating,wolfer2024optimistic}.

In the case of Markov chains, the estimation of the mixing time $t_{mix}$, the spectral gap and the pseudo-spectral gap was also studied thoroughly in the past years when the state space is finite~\citep{hsu2015mixing,levin2016estimating,10.1214/18-AAP1457,wolfer2022estimatingmixingtimeergodic,wolfer2024improved}.

\subsection{Contributions and Organization of the Paper}

In this paper, we derive a PAC-Bayes bound for Markov chains that depends on a spectral quantity called the pseudo-spectral gap: $\gamma_{ps}$ (the definition will be given below). The main tool in the proof is a Bernstein inequality for Markov chains due to~\citet{paulin2015}. Recently,~\cite{wolfer2024improved} derived estimators of $\gamma_{ps}$ for finite-state Markov chains, together with confidence intervals. From this we derive empirical versions of our PAC-Bayes bound. We also provide an example in which such an empirical bound can be obtained while the state space is infinite.

In the end of this introduction, we introduce the notation used in the paper. We also remind important notions on Markov chains and define the pseudo-spectral gap $\gamma_{ps}$. In Section~\ref{section:nonempirical}, we provide the (non-empirical) version of the PAC-Bayes bound for Markov chains. Then, in Section~\ref{section:empirical}, we study various settings in which this bound can be made empirical. Finally, in Section~\ref{section:application}, we exemplify our results in the problem of learning the best predictor in a finite set and provide numerical evaluations of the bound in this context. The proofs are gathered in the supplement, together with additional experiments and a discussion on other possible approaches.

\subsection{Problem Formulation}

Let $\mathcal{U}$ denote the object space and $\mathcal{Y}$ the label space. Suppose, we are given object $\times$ label observations $(U_1,Y_1)$, $(U_2,Y_2)$, $\ldots$, $(U_n,Y_n)$. Usually, it is assumed that the pairs $(U_t,Y_t)$ are i.i.d. from a distribution $Q$. In such a setting, $U_1,\dots,U_n$ would be i.i.d. from the first marginal distribution $Q_U$ of $Q$.
Here, we want to allow some temporal dependence between the objects $(U_t)$, so we will not assume that they are independent. Instead, we will assume they form a stationary Markov chain (the definition will be reminded below). As in the i.i.d. case, we will assume that the label $Y_t$ depends \textit{only} on $U_t$. In other words, there is a regular condition probability distribution $Q(\cdot,\cdot)$ such that the distribution of $Y_t$ given $(U_1,\dots,U_t)$ is given by $Q(U_t,\cdot)$.
For short, let $\mathcal{S} = ((U_1,Y_1),\dots,(U_n,Y_n))$ denote the sample.

We consider a parametrized set of predictors $\mathcal{F}=\{f_\theta : \mathcal{U}\overset{}{\to} \mathcal{Y}\ |\: \theta\in\Theta\}$. To measure the prediction error, we use a loss function $\ell:\mathcal{Y}^2\overset{}{\to} \mathcal{\mathbb{R}}_+$, which is assumed to be bounded throughout the paper $\ell(\cdot,\cdot)\leq c$.  
To measure the accuracy of the prediction, we will use the classical notion of risk:
$$R(\theta)=\mathbb{E}_{\mathcal{S}}\left[\frac{1}{n}\sum_{t=1}^{n}\ell\left(f_\theta(U_t),Y_t\right) \right] $$
(where $\mathbb{E}_{\mathcal{S}}$ denotes the expectation with respect to the sample, see Subsection~\ref{subsec:notations} below),
and the empirical risk
$$r(\theta)=\frac{1}{n}\sum_{t=1}^{n}\ell\left(f_\theta(U_t),Y_t\right).$$

\subsection{Definitions and Reminders on Markov Chains}
Let $(\Omega,\mathcal{A},\mathbb{P})$ be a probability space, where $\Omega$ is a set equipped with a $\sigma$-algebra $\mathcal{A}$ and a probability measure $\mathbb{P}$. A $\mathcal{U}$-valued sequence of random variables $\{U_t\}_{t \geq 1}$ is said to be a Markov chain on $(\Omega,\mathcal{A},\mathbb{P})$ if it satisfies the Markov property $\mathbb{P}(U_t \in A \mid U_{t-1}, U_{t-2}, \ldots, U_1) = \mathbb{P}(U_t \in A \mid U_{t-1})$ for any $t \geq 1$ and $A \in \mathcal{A}$. We will also assume the chain to be homogeneous: $ \mathbb{P}(U_t \in A \mid U_{t-1}=u)$ does not depend on $t$. It is then common to introduce the notation $P(u,A):= \mathbb{P}(U_t \in A \mid U_{t-1}=u)$. $P$ is called the transition kernel of the chain. Note that for every $u$, $P(u,\cdot)$ is a probability measure.

We say that a distribution $\pi$ on $\mathcal{U}$ is a stationary distribution for the chain if
$$
\int_{u\in\mathcal{U}} \pi(d u) P(u, {\rm d}x)=\pi({\rm d}x).
$$
If $U_1\sim \pi$, then any $U_t\sim\pi$ for $t\geq 1$ and the chain is said to be stationary. We will essentially work with stationary chains.
\subsubsection{Asymptotic Behavior and Ergodicity}
In this paper, we will work under the assumption that $P$ is ergodic, which in particular implies it has a unique invariant distribution $\pi$. We will moreover assume it satisfies a spectral property defined below. This property will be enough to ensure that, for $t$ large enough, the distribution of $U_t$ is close enough to $\pi$ regardless of the distribution of $U_1$. First, it will be helpful to remind a stronger, but more classical condition.

A chain is called uniformly ergodic with a rate $0\leq\rho<1$, if there exists $C>0$ such that
$$
\sup_{u\in\mathcal{U}}\|P^k(u,\cdot)-\pi(\cdot)\|_{TV}\leq C \rho^k .
$$
For example, when $\mathcal{U}$ is a finite set, any chain that is irreducible has a unique invariant distribution $\pi$. If it is also aperiodic, then it is uniformly ergodic, see for example~\citet{douc2018markov}.
\begin{defi}
    The mixing time $t_{\text{mix}}$ of the chain is defined by $t_{\text{mix}}:=t_{\text{mix}}\left(1/4\right)$ where
$$
t_{\text{mix}}(\varepsilon):=\inf\left\{k: \sup_{u\in\mathcal{U}}\|P^k(u,\cdot)-\pi(\cdot)\|_{TV} \leq\varepsilon\right\}.
$$
\end{defi}
The mixing time of a Markov chain measures how quickly the chain forgets its initial state and becomes close to its stationary distribution. If $t_{\text{mix}}=+\infty$, then the chain is not uniformly ergodic.



\subsubsection{Pseudo-Spectral Gap}
Let $L^2(\pi)$ be the Hilbert space of complex valued measurable functions on $\Omega$ that are square integrable with respect to $\pi$. Consider $L^2(\pi)$ equipped with the inner product $\langle f, g\rangle_\pi=\int f g^* \mathrm{~d} \pi$, and norm $\|f\|_{2, \pi}:=\langle f, f\rangle_\pi^{1 / 2}=\left(\mathbb{E}_\pi\left[f^2\right]\right)^{1 / 2}$, then $P$ defines a linear operator on $L^2(\pi)$ given by
$$ (P f)(u):=\mathbb{E}_{V \sim P(u, \cdot)}[f(V)] = \int_v P(u,{\rm d} v) f(v)  $$
for any $f\in L^2(\pi)$. The operator $P$ acts on measures to the left: for a probability measure $\nu$, $\nu P$  is also a probability measure given by $\nu P(A):=$ $\int_{u} P(u, A) \nu(\mathrm{d} u)$ for every $A\in\mathcal{A}$. \\

As mentioned above, we assume $P$ admits a unique invariant $\pi$. Such a kernel is said to be reversible if
$$\pi({\rm d}u)\,P(u,{\rm d}v)=\pi({\rm d}v)\,P(v,{\rm d}u). $$

Reversibility of the chain is equivalent to the linear operator $P$ being self-adjoint on $L^2(\pi)$. When $P$ is not self-adjoint (or the chain is not reversible), the chain is said to be non-reversible. We define the time reversal kernel $P^*$ of $P$ by
$$P^\ast(u,{\rm d}v)
            \;:=\;
            \frac{\pi({\rm d}v)\,P(v,{\rm d}u)}{\pi({\rm d}u)}, $$
and the linear operator $P^*$ is the adjoint of $P$ in $L^2(\pi)$. In particular, if $P$ is reversible, then $P^* = P$.
Suppose $\pi$ is the stationary distribution of the chain, and $I$ is the identity operator. 
The spectrum ${\rm sp}(P)$ is defined as the set of all $\tau \in\mathbb{C}\backslash 0$ such that $( \tau I-P)^{-1}$ does not exist or is not a bounded linear operator in $L^2(\pi)$.

For a transition kernel $P$,  $\tau \in {\rm sp}(P)$ satisfy $|\tau |\leq 1$. Moreover, $\tau =1$ is necessarily an eigenvalue of $P$, and thus $1\in {\rm sp}(P)$. When $P$ is in addition reversible, ${\rm sp}(P) \subset \mathbb{R}$. 
In this case, if the multiplicity of the eigenvalue $\tau =1$ is $1$,  the spectral gap of $P$ is defined as
$$ \gamma(P) := 1 - \sup\{\tau \in {\rm sp}(P): \tau \neq 1  \}. $$
 When the multiplicity of the eigenvalue $\tau =1$ is larger than $1$, we define $\gamma(P)=0$.

 Another spectral characterization of a chain is the notion of a pseudo-spectral gap, proposed by \citet{paulin2015}. It is more general, as it will allow us to consider chains that are not reversible.
\begin{defi}
The pseudo-spectral gap is defined by
          $$
            \gamma_{ps}(P)\;:=\; \max_{k\ge1}\; \left\{\frac{\gamma\!\left((P^\ast)^{\,k}P^{\,k}\right)}{k}\right\}.
          $$
\end{defi}
As the transition kernel of the observations will always be $P$, we can safely write $\gamma_{ps}$ instead of $\gamma_{ps}(P)$.
In this paper, our main assumption on the observations is that they form a Markov chain with positive pseudo-spectral gap: $\gamma_{ps}>0$.

We mention that this condition is more general than the classical uniform ergodicity condition. Indeed, Proposition 3.4 of~\cite{paulin2015} states that, if a Markov chain is uniformly ergodic, and thus has $t_{mix}<+\infty$, then
$$
\gamma_{ps} \geq \frac{1}{2t_{mix}}>0.
$$ 
There are many examples of chains with $\gamma_{ps} >0$ that are not uniformly ergodic, such as the AR(1) process in Subsection~\ref{subsection:ar1} below.

Thanks to a result of \citet{davydov1968convergence} that we remind in Appendix C, for a uniformly ergodic chain, the $\varphi$-mixing coefficients $\varphi(k)$ decrease to $0$ exponentially fast, while for a chain which is not uniformly ergodic, $\varphi(k)$ does not converge to $0$. Thus, the condition $\gamma_{ps}>0$ is also weaker than the condition $\varphi(k)\rightarrow 0$. We are not aware of a direct relation between $\gamma_{ps}$ and the $\beta$-mixing coefficients, see the discussion by~\citet{wolfer2024optimistic}.

\subsection{Notations}
\label{subsec:notations}

Expectation and probability with respect to the sample will be denoted by $\mathbb{E}_{\mathcal{S}}$ and $\mathbb{P}_{\mathcal{S}}$ respectively. In PAC-Bayes bounds, we also consider expectation and probabilities with respect to some random parameter $\theta$ sampled from various probability distributions. So it is important to keep the distribution in the notation. When $\theta$ is sampled from some $\rho$, we will respectively write $\mathbb{E}_{\theta\sim\rho}$ and $\mathbb{P}_{\theta\sim\rho}$ for the expectation and probability with respect to $\theta$. Rigorously,
$$ \mathbb{E}_{\theta\sim\rho}[f(\theta)] = \int f(\theta) \rho({\rm d}\theta)  $$
(when this integral is well-defined).
We let $\mathcal{P}(\Theta)$ denote the set of all probability distributions on $\Theta$ (equipped with a $\sigma$-field). Given $\nu_1,\nu_2\in\mathcal{P}(\Theta)$ we let $KL(\nu_1\|\nu_2)$  denote the Kullback-Leibler divergence between $\nu_1$ and $\nu_2$. PAC-Bayes bounds involve a reference measure in $\mathcal{P}(\Theta)$ called the prior, we will let $\mu$ denote the prior throughout the paper. For an integer $K$, $[K]:=\{1,2,\dots,K\}$.

\section{PAC-BAYES BOUNDS FOR MARKOV CHAINS}
\label{section:nonempirical}

We first state a (non-empirical) PAC-Bayes bound in this setting. The proof follows the same steps of the classical PAC-Bayes bound of~\citet{catoni2003} in the i.i.d. setting. To handle Markov data, we use known concentration results for Markov chains from~\citet{paulin2015}.

\begin{thm}\label{bern1}
    Assume $\{U_t\}_{t =1}^n$ is a stationary Markov chain with pseudo-spectral gap $\gamma_{ps}>0$. Then for any constants $0<\lambda<\frac{n}{10}$, $\delta\in(0,1)$, and prior $\mu\in\mathcal{P}(\Theta)$,
\begin{multline*}
\mathbb{P}_{\mathcal{S}}\Bigg(\forall\rho\in\mathcal{P}(\Theta),\:\: \mathbb{E}_{\theta\sim\rho}\left[R(\theta)\right]\leq\mathbb{E}_{\theta\sim\rho}\left[r(\theta)\right]\\
+ \frac{2\lambda c^2\left(1+\frac{1}{n \gamma_{ps}}\right)}{n - 10\lambda}+ \frac{KL(\rho||\mu)+\log \frac{1}{\delta}}{\lambda \gamma_{ps}}  \Bigg)\geq1-\delta.
\end{multline*}
\end{thm}
Observe the effect of $\gamma_{ps}$ on the bound: the larger $\gamma_{ps}$, the tighter the bound is. However, when $\gamma_{ps}\rightarrow 0$, the bound explodes to infinity. Prediction is easier with a larger $\gamma_{ps}$.

In practice, if we observe data generated by an unknown Markov process, $\gamma_{ps}$ is usually unknown. A naive approach is to assume a lower bound on $\gamma_{ps}\geq \gamma_0$, say $\gamma_{ps}>0.1$ (this is similar to the \textit{a priori} upper bounds assumed on mixing coefficients in the previous works mentioned in the introduction).
This approach is problematic for two reasons: first, if $\gamma_{ps}=0.05$, then our generalization bound is wrong. Moreover, if $\gamma_{ps}=0.9$, our bound is correct, but it is also excessively pessimistic.

An alternative is to assume $\gamma_{ps}\geq \gamma_0 = 1/n^{a}$ with $a\in(0,1)$: such an assumption will always be satisfied for large enough sample size $n$.
Under such an assumption, we can simply upper bound $1/\gamma_{ps}$ by $n^a$ in the theorem. For the term $1+1/(n\gamma_{ps})\leq 1+1/n^{1-a}\leq 2$, this is actually not a bad upper bound. The problem comes from
$$
\frac{KL(\rho||\mu)+\log \frac{1}{\delta}}{\lambda \gamma_{ps}}
$$
which will change the order of magnitude of the bound.
It would of course be far better to replace $\gamma_{ps}$ by a consistent estimator $\widehat{\gamma}_{ps}$. If we can give an accurate upper bound on  $1/\gamma_{ps}$ in terms of  $1/\widehat{\gamma}_{ps}$, we will obtain an empirical PAC-Bayes bound. This is the object of the next section.

\section{EMPIRICAL PAC-BAYES BOUNDS}
\label{section:empirical}

Assuming we have an estimator of the pseudo-spectral gap $\gamma_{ps}$, we can state the following corollary of Theorem~\ref{bern1}.
\begin{cor}
  \label{cor: bern_pseudo_added}
  Under the conditions of Theorem~\ref{bern1}, fix $a\in(0,1)$ and assume that $n$ is large enough to ensure $n \geq 1/\gamma_{ps}^{1/a}$.
Assume we have an estimator $\widehat{\gamma}_{ps}$ of $\gamma_{ps}$ such that, for any $\varepsilon>0$,
\begin{align}\label{general gamma hat}
    \mathbb{P}\left( \left| \frac{\widehat{\gamma}_{ps} }{\gamma_{ps}} - 1  \right| \leq \varepsilon  \right)\geq 1-\alpha(n,\gamma_{ps},\varepsilon).
\end{align}
Then, we have
\begin{multline*}
\mathbb{P}_{\mathcal{S}}\Bigg(\forall\rho\in\mathcal{P}(\Theta),\:\: \mathbb{E}_{\theta\sim\rho}\left[R(\theta)\right]\leq\mathbb{E}_{\theta\sim\rho}\left[r(\theta)\right]\\
+ \frac{2\lambda c^2\left(1+\frac{1}{n^{1-a}}\right)}{n - 10\lambda}+ \frac{KL(\rho||\mu)+\log \frac{1}{\delta}}{\lambda \widehat{\gamma}_{ps}}\left(1+\varepsilon \right)  \Bigg)
\\
\geq 1-\delta - \alpha(n,\gamma_{ps},\varepsilon).
\end{multline*}
\end{cor}
A condition for the corollary to be actually useful is that $\alpha(n,\gamma_{ps},\varepsilon)$ is a nonincreasing function of $\gamma_{ps}$ that satisfies $\alpha(n,n^{-1/a},\varepsilon)\rightarrow 0$ when $n\rightarrow\infty$. Indeed, in this case,
\begin{equation}
\label{equalpha}
 \alpha(n,\gamma_{ps},\varepsilon) \leq  \alpha(n,n^{-1/a},\varepsilon) \xrightarrow[n\rightarrow\infty]{} 0.
\end{equation}
 
\subsection{First Example of Estimation of $\gamma_{ps}:$ the Finite State Space Case}

For an ergodic Markov chain on a finite state-space,  say ${\rm card}(\mathcal{U})=d$, \citet{wolfer2022estimatingmixingtimeergodic} provided an estimator for the pseudo-spectral gap given by the following formula:
\begin{align}
    \label{spec-estim}
    \widehat{\gamma}_{ps,[K]}= \max_{k \in [K]}\left\{\frac{\gamma\left(\left(\hat{P}^{\dagger}\right)^k \hat{P}^k\right)}{k}\right\}
\end{align}
where $\hat{P}=\hat{P}\left(U_1,U_2,\cdots, U_n\right)$ is a natural empirical estimator for $P$ and $K$ a positive integer.

Rewriting their result in a way that matches Corollary~\ref{cor: bern_pseudo_added}, we obtain the following proposition (the proof is provided in the appendix).
\begin{prp}\label{prp:gamma_ratio}
Under the conditions of Theorem~\ref{bern1}, assuming the chain $(U_t)$ is ergodic and the state-space is finite, that is ${\rm card}(\mathcal{U})=d$, for any $\varepsilon>0$, the estimator $\widehat{\gamma}_{ps}:=\widehat{\gamma}_{ps,[K]}$ given by \ref{spec-estim} with $K=\lceil 2/\varepsilon \rceil$ we have
\begin{multline}
    \mathbb{P}_{\mathcal{S}}\left(\left|\frac{\widehat{\gamma}_{ps}}{\gamma_{ps}}-1\right|\geq\varepsilon\right)
    \\
    \leq \frac{C_{ps}d}{\varepsilon \gamma_{ps} \sqrt{\pi_*}} e ^ { -n\varepsilon^2 \gamma_{ps}^2 \pi_*\min\left\{\gamma_{ps},\frac{1}{C(P)}\right\} }
\end{multline}
where $C(P)=\|P\|_{\pi}\min\{d,\|P\|_{\pi}\}$, with $\|P\|_{\pi}=\max\{\pi(i) / \pi(j) , i,j\in[d]^2 \}$, and $\pi_* = \min\{\pi(i),i\in[d]\} $.
\end{prp}
Observe that, as we assume that the chain is ergodic, $\pi_* = \min\{\pi(i),i\in[d]\}>0$. However, $\pi_*$ can be arbitrarily small, which leads to less confident estimation of $\gamma_{ps}$. Then, note that, by taking $a$ large enough,~\eqref{equalpha} is satisfied.

\subsection{Example of Estimation of $\gamma_{ps}$ in the Infinite Case}
\label{subsection:ar1}

In the finite case, we estimated the pseudo-spectral gap without strong assumptions on $P$. As argued by~\cite{wolfer2022estimatingmixingtimeergodic}, this is not possible for infinite Markov chains, even in the countable case. Intuitively, this can be understood from Proposition~\ref{prp:gamma_ratio}: when the state space is countably infinite, we have necessarily $\pi_*=0$, and thus, the statement of the proposition becomes vacuous.

Obtaining empirical bounds is feasible, however, only by imposing strong restrictions on $P$. In this subsection, we illustrate this fact in the situation where the inputs are sampled from an autoregressive process on the real line. That is, we assume that $(U_t)_{t \geq 1}$ is a stationary process with
\begin{equation}
\label{eq:AR1}
U_t = a U_{t-1} + \zeta_t
\end{equation}
where $-1<a<1$ and the $\zeta_t$ are i.i.d. from $\mathcal{N}(0,1)$. In other words, $P(x,\cdot) = \mathcal{N}(ax,1)$. Such a process is known to be ergodic, but non-uniformly ergodic: $t_{mix}=+\infty$. The following propositions show that its pseudo-spectral gap has a simple form and can be estimated with confidence.
\begin{prp}
\label{prp:ar1:1}
Let $(U_t)_{t \geq 1}$ be a stationary AR(1) process, defined by~\eqref{eq:AR1}, then its pseudo-spectral gap is given by
$$\gamma_{ps} = 1-a^2= \frac{1}{{\rm Var}(U_1)}.$$
\end{prp}
\begin{prp}
\label{prp:ar1:2}
Let $(U_t)_{t \geq 1}$ be a stationary AR(1) process, defined by~\eqref{eq:AR1}, then for the estimator $\widehat{\gamma}_{ps}$ given by
\begin{align}\label{AR estimate}\widehat{\gamma}_{ps} := \min\left\{\frac{1}{\frac{1}{n} \sum_{t=1}^n U_t^2},1\right\},
\end{align}
it holds
$$
\mathbb{P}_{\mathcal{S}}\left(
\left|
\frac{\widehat{\gamma}_{ps}}{\gamma_{ps}}-1
\right|
\leq \varepsilon
\right) \geq 1-\exp\left( \frac{9}{4}-\frac{n\varepsilon^2 \gamma_{ps}^3}{2304} \right).
$$
\end{prp}
In other words, 
$$
\mathbb{P}_{\mathcal{S}}\left(
\left|
\frac{\widehat{\gamma}_{ps}}{\gamma_{ps}}-1
\right|
\leq  \frac{24}{\gamma_{ps}^{3/2}} \sqrt{\frac{9+4\log \frac{1}{\delta}}{n}  }
\right) \geq 1-\delta.
$$
The proof relies on more general results on the estimation of variances and covariances of time series in~\cite{nakakita2024dimension}. Combining Proposition~\ref{prp:ar1:2} and Corollary~\ref{cor: bern_pseudo_added}, we obtain the following result.
\begin{cor}
Let $(U_t)_{t \geq 1}$ be a stationary AR(1) process, defined by~\eqref{eq:AR1}, and $\widehat{\gamma}_{ps}$ defined by~\eqref{AR estimate}, then we have
\begin{multline*}
\mathbb{P}_{\mathcal{S}}\Bigg(\forall\rho\in\mathcal{P}(\Theta),\:\: \mathbb{E}_{\theta\sim\rho}\left[R(\theta)\right]
\\
\leq\mathbb{E}_{\theta\sim\rho}\left[r(\theta)\right] + \frac{2\lambda c^2\left(1+\frac{1}{\gamma_{ps} n}\right)}{n - 10\lambda}
\\
 + \frac{KL(\rho||\mu)+\log \frac{1}{\delta}}{\lambda \widehat{\gamma}_{ps}}\left(1 + \frac{24}{\gamma_{ps}^{3/2}} \sqrt{\frac{9+4\log \frac{1}{\delta}}{n}  } \right)  \Bigg)
 \\
 \geq 1-2 \delta.
\end{multline*}
\end{cor}
For example, when $n$ is large enough to ensure $n\geq 1/\gamma_{ps}^{4} $ then
\begin{multline*}
\mathbb{P}_{\mathcal{S}}\Bigg(\forall\rho\in\mathcal{P}(\Theta),\:\: \mathbb{E}_{\theta\sim\rho}\left[R(\theta)\right]
\\
 \leq\mathbb{E}_{\theta\sim\rho}\left[r(\theta)\right]+ \frac{2\lambda c^2\left(1+\frac{1}{n^{3/4}}\right)}{n - 10\lambda}
\\
 + \frac{KL(\rho||\mu)+\log \frac{1}{\delta}}{\lambda \widehat{\gamma}_{ps}}\left(1 +\frac{24}{n^{1/8}} \sqrt{9+4\log \frac{1}{\delta}  } \right)  \Bigg)
 \\
 \geq 1-2 \delta.
\end{multline*}
The tools developed in~\cite{nakakita2024dimension} allow to tackle more general situations, such as multivariate $U_t$'s and the case where the variance of $\zeta_t$ is unknown.

\subsection{Optimization with Respect to $\lambda$ and Oracle Bounds}

We discuss briefly here how to tune the parameter $\lambda$ in the PAC-Bayes bound and how we can obtain oracle bounds. The procedure is relatively standard, so we provide only the bare minimum, together with references for more details.

Given a finite grid $\Lambda=\{\lambda_1,\dots,\lambda_L\}$ of possible values for $\lambda$, we can perform a union bound on Theorem~\ref{bern1}. We obtain:
\begin{multline*}
\mathbb{P}_{\mathcal{S}}\Bigg(\forall\rho\in\mathcal{P}(\Theta), \exists\lambda\in\Lambda\:\: \mathbb{E}_{\theta\sim\rho}\left[R(\theta)\right]\leq\mathbb{E}_{\theta\sim\rho}\left[r(\theta)\right]
\\
+\frac{2\lambda c^2\left(1+\frac{1}{n \gamma_{ps}}\right)}{n - 10\lambda}+ \frac{KL(\rho||\mu)+\log \frac{L}{\delta}}{\lambda \gamma_{ps}}  \Bigg)\geq 1-\delta.
\end{multline*}

\begin{defi}
We put
$$\hat{\rho}=\argmin_{\rho}\left[\mathbb{E}_{\theta\sim\rho}\big[r(\theta)\big]+ B\left(\rho,\frac{(1+\varepsilon)}{\widehat{\gamma}_{ps}}\right)\right]$$
where, for any probability distribution $\nu\in\mathcal{P}(\Theta)$ and any real number $u>0$,
\begin{multline*}
B\left(\nu,u\right) := \min_{\lambda\in\Lambda} \Bigg\{\frac{2\lambda c^2\left(1+\frac{u}{n} \right)}{n - 10\lambda}
\\+ u \frac{KL(\nu||\mu)+\log \frac{L}{\delta}}{\lambda}\Bigg\}.
\end{multline*}
\end{defi}

The excess risk of $\hat{\rho}$ is upper bounded in the following theorem.
\begin{thm}\label{excess}
Under the conditions of Theorem~\ref{bern1},
\begin{equation*}
\mathbb{E}_{\theta\sim\hat{\rho}}\big[R(\theta)\big]
\leq
\inf_{\rho}
\Biggl\{
\mathbb{E}_{\theta\sim\hat{\rho}}\big[R(\theta)\big] + 2 B\left( \rho,\frac{1+\varepsilon}{\gamma_{ps}-\varepsilon}  \right)
\Biggr\}
\end{equation*}
with probability at least $1-2\delta-\alpha(n,\gamma_{ps},\varepsilon)$.
\end{thm}
\begin{remark} A well-chosen grid will contain a $\lambda$ of the order of
$$
c \sqrt{ \frac{2 n (1+\varepsilon) KL(\rho\|\mu)}{\gamma_{ps} - \varepsilon}  }
$$
which gives $B\left( \rho,\frac{1+\varepsilon}{\gamma_{ps}-\varepsilon}  \right)$ of the order of
$$
2 c \sqrt{ \frac{ 2 (1+\varepsilon) KL(\rho\|\mu) }{ n(\gamma_{ps}-\varepsilon) } },
$$
we refer the reader to Section 2.1.4 in~\cite{Alquier_2024} for more details on the construction of the grid.
\end{remark}
\begin{remark}
As in Corollary~\ref{cor: bern_pseudo_added}, we could exemplify the theorem in the case where $\gamma_{ps}\geq 1/ n^a$. However, the constraint $\varepsilon<\gamma_{ps}$ will require one to take $\varepsilon< 1/ n^a$. For example, with $\varepsilon = 1/(2n^a)$ we obtain:
\begin{equation*}
\mathbb{E}_{\theta\sim\hat{\rho}}\big[R(\theta)\big]
\leq
\inf_{\rho}
\Biggl\{
\mathbb{E}_{\theta\sim\hat{\rho}}\big[R(\theta)\big] + 2 B\left( \rho, 2n^a + 1 \right)
\Biggr\}
\end{equation*}
with probability at least $1-2\delta-\alpha(n,1/n^a,1/2n^a)$. This also requires to check that $\alpha(n,1/n^a,1/2n^a)\rightarrow 0$. For $a$ small enough, this is straightforward in the two examples we developed above.
\end{remark}

\section{APPLICATION: FINITE SET OF PREDICTORS}
\label{section:application}

In this section, we exemplify the approach of Sections~\ref{section:nonempirical}
and~\ref{section:empirical} in the case where the set of predictors is finite: ${\rm card}(\Theta)=M<+\infty$. This case was studied extensively in the machine learning literature. We believe it is also of pedagogical interest as the bound takes a simpler form in this situation, the reader will also observe that the optimization with respect to $\lambda$ is more explicit. We will then assess the tightness of the bound on simulated data.

\subsection{PAC-Bayes Bound with a Finite $\Theta$}

We consider the posterior $\rho$ that minimizes the expected empirical risk. This $\rho$ corresponds to the Dirac mass on the empirical risk minimizer $\hat{\theta}_{\text{ERM}}=\argmin_\theta r(\theta)$, since 
\begin{align}\label{dirac_erm}
\inf_\rho\mathbb{E}_{\theta\sim\rho}\big[r(\theta)\big]=\inf_\rho\left[\sum_{\theta_i}\rho(\theta_i)r(\theta_i)\right]=r(\hat{\theta}_{\text{ERM}}).
\end{align}

The next theorem, which is proven using the PAC-Bayes bound of Theorem~\ref{bern1}, provides a generalization bound on $\hat{\theta}_{\text{ERM}}$.
\begin{thm}\label{thm: finite_erm theoretical}
Fix $\varepsilon>0$.
 Let $\{U_t\}_{t = 1}^{n}$ be a stationary Markov chain with pseudo-spectral gap $\gamma_{ps}>0$. Suppose $\text{card}(\Theta)=M<\infty$, and $\mu$ is the uniform prior on $\Theta$, then for any $\delta\in(0,1)$, as soon as $n$ is large enough to ensure
 $$n > \frac{ 50(1+\varepsilon) \log \frac{M}{\delta}}{\varepsilon^2 c^2\gamma_{ps}\left(1+\frac{1}{\gamma_{ps}n}\right)},
$$
we have
 \begin{multline*}
\mathbb{P}_{\mathcal{S}}\Bigg( R(\hat{\theta}_{\text{ERM}})\leq r(\hat{\theta}_{\text{ERM}}) 
\\          +         \sqrt{\frac{ 8(1+\varepsilon) c^2 \log \frac{M}{\delta} }{ \gamma_{ps}n }\left(1+\frac{1}{\gamma_{ps}n}\right)} \; \Bigg)\geq 1-\delta.
\end{multline*}         
\end{thm}

\begin{remark} The main difference between this bound and similar results in the i.i.d. case (like Example 2.1 of~\citet{Alquier_2024}) is that $n$ is replaced by $n\gamma_{ps}$. Thus, we can think of $n\gamma_{ps}$ as an "effective sample size". When $\gamma_{ps}$ is close to one, Markov observations are almost as informative as i.i.d. observations. 
\end{remark}
\begin{remark}
     When the initial prior $\mu$ is not uniform, the bound in the statement still holds after replacing 
  $\log \!\frac{M}{\delta}$ with 
  $\log \!\frac{1}{\mu(\hat{\theta}_{\mathrm{ERM}})\,\delta}$. All other terms in the bound remain unchanged.
\end{remark}

We now mimic what was done in Section~\ref{section:empirical} in the general case, to make this bound empirical, by using an estimator $\widehat{\gamma}_{ps}$.
\begin{cor}
  \label{cor: bern_pseudo_added:finite_case}
 Under the conditions of Theorem~\ref{thm: finite_erm theoretical}, fix $a\in(0,1)$ and assume that $n$ is large enough to ensure $n \geq 1/\gamma_{ps}^{1/a}$.
Assume we have an estimator $\widehat{\gamma}_{ps}$ of $\gamma_{ps}$ such that
\begin{align}\label{finite gamma hat}
    \mathbb{P}_{\mathcal{S}}\left( \left| \frac{\widehat{\gamma}_{ps} }{\gamma_{ps}} - 1  \right| \leq \varepsilon  \right)\geq 1-\alpha(n,\gamma_{ps},\varepsilon),
\end{align}
then
\begin{multline*}
\mathbb{P}_{\mathcal{S}}\Bigg( R(\hat{\theta}_{\text{ERM}})\leq r(\hat{\theta}_{\text{ERM}}) 
\\                   
+   \sqrt{\frac{ 8 c^2 \log \frac{M}{\delta} }{ \widehat{\gamma}_{ps}n }(1+\varepsilon)^2\left(1+\frac{1}{n^{1-a}}\right)} \Bigg)
\\
\geq 1-\delta-\alpha(n,\gamma_{ps},\varepsilon).
\end{multline*}      

\end{cor}

\subsection{Experiments}

In this section we assess the accuracy of the empirical bound of Corollary~\ref{cor: bern_pseudo_added:finite_case}. That is, we sample trajectories of Markov chains\footnote{The code is available online: \url{https://github.com/v-ahe/empirical-pac-bayes-markov/tree/main}} with various state spaces and transition kernels $P$. On the contrary to a real-life situation, where $P$ would be unknown, we can compute $\gamma_{ps}$ numerically here. This allows to compare the non-empirical bound of Theorem~\ref{thm: finite_erm theoretical} to the empirical bound of Corollary~\ref{cor: bern_pseudo_added:finite_case}, using the estimator $\widehat{\gamma}_{ps}$ of~\citet{wolfer2022estimatingmixingtimeergodic}, see~\eqref{spec-estim}.

\textbf{Setting:} we will consider $\mathcal{U}=[d]$ for various values of $d$: 4, 10, 20, 50 and 100. We work with a binary classification problem, where $\mathcal{Y}=\{0,1\}$, and our set of predictors are simply thresholds $f_{\theta}(u)=\mathbf{1}(u\geq \theta)$ for $\theta\in\Theta=[d]$, that is, ${\rm card}(\Theta)=d$ in Theorem~\eqref{thm: finite_erm theoretical} and Corollary~\ref{cor: bern_pseudo_added:finite_case}. We consider various with various sample sizes: $n\in\{10,100,1000,10000\}$.

In order to study the behavior of the bound under various values of $\gamma_{ps}$, we ran simulations for a wide range of transition kernels. We designed these kernels as follows: we first fixed a transition kernel $P$ with $\gamma_{ps}(P)\simeq 0$, and a transition kernel $Q$ with $\gamma_{ps}(Q) = 1$. This is easily obtained with a rank one $Q=1^T\cdot\pi$ where $1^T=(1,\dots,1)$ and $\pi$ is the invariant distribution of $P$; we refer the reader to the supplement for the exact definition of $P$.
We then defined the kernels
\begin{equation}
\label{interpol}
R_t := t P + (1-t) Q
\end{equation}
for an interpolation parameter $t\in[0,1]$, and ran experiments for each $t\in\{k/20, 0 \leq k \leq 20\}$. In each case, we simply sample trajectories $(U_1,\dots,U_n)$ from the kernel $R_t$, then, the 
\begin{figure}[t]
\centering
\includegraphics[width=1.00\linewidth]{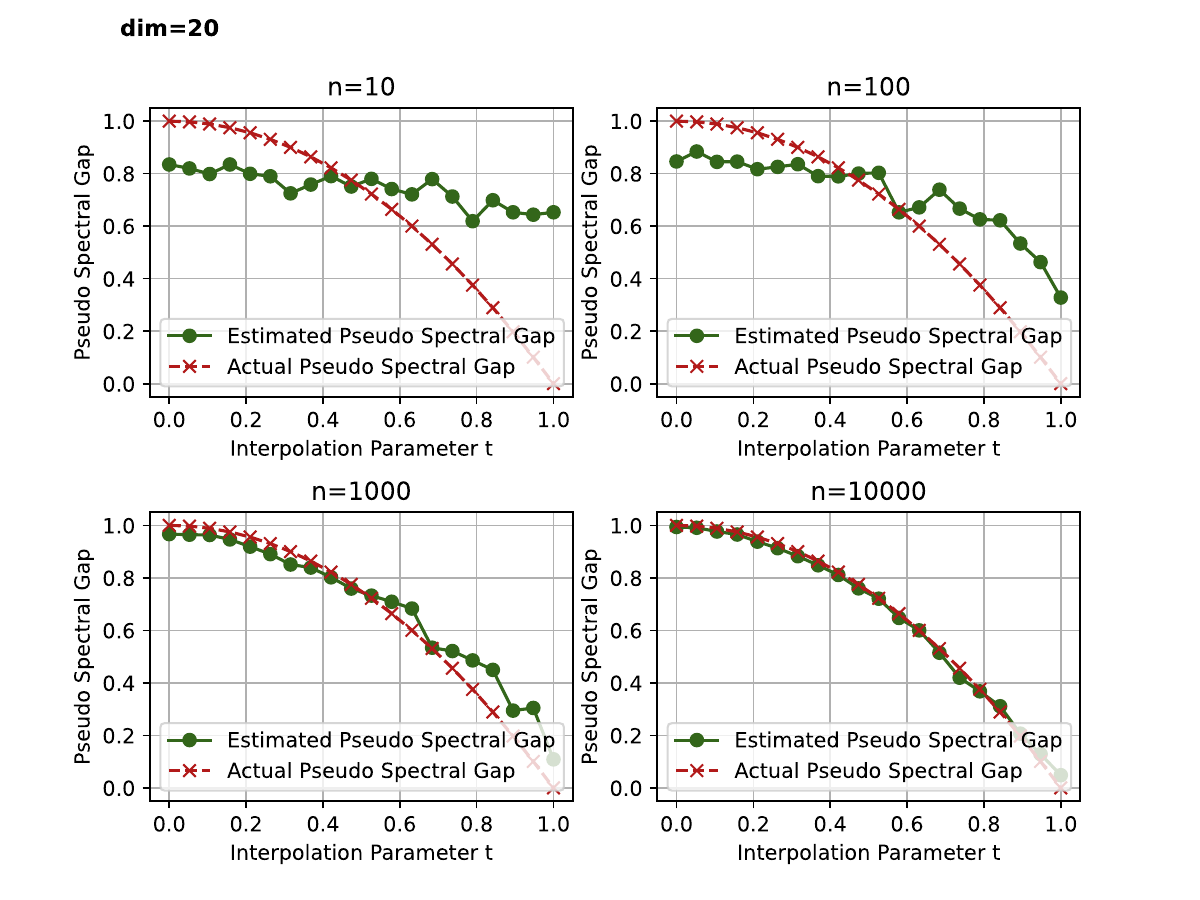}
\vspace*{1px}
\caption{Estimation of $\gamma_{ps}(R_t)$ when $d=20$. In red, the actual values of $\gamma_{ps}(R_t)$, as a function of our interpolation parameter $t$, see~\eqref{interpol}. In green, the value of the estimator $\widehat{\gamma}_{ps}$.}
\label{figure: gamma_20}
\end{figure}

\textbf{Tuning parameters:} the parameter $K$ in the definition of $\hat{\gamma}_{ps,[K]}$ was set as $20$ in all our experiments\footnote{In most cases, we found that the maximum was reached for $k=1$. Note however that, in the $d=4$ setting, in a couple of cases the maximum was reached for $k=4$. It can be understood intuitively why this is the case. Assume for simplicity that $P$ is reversible, that is $P^{\star}=P$, and let $\lambda$ denote the largest eigenvalue different from $1$. Then $\gamma((P^{\star})^k P^k) = \gamma(P^{2k}) = (1-\lambda^2)^{k}$ and thus $\gamma((P^{\star})^k P^k)/k = (1-\lambda^2)^{k}/k$, which is obviously decreasing in $k$. Explicit examples where the maximum is reached for $k>1$ are necessarily non-reversible. Such example can be found in~\cite{paulin2016mixing}.}.
The estimator $\hat{P}$ of the transition matrix of~\cite{wolfer2024improved} involves a smoothing parameter $\alpha$ that was set to $1$. The empirical PAC-Bayes bound in Corollary~\ref{cor: bern_pseudo_added:finite_case} depends on two parameters $a$, set to $0.1$, and $\varepsilon$, set to $0.1/n^{1/3}$.

\textbf{Checking the estimator $\widehat{\gamma}_{ps}$:} first, we ran sanity checks on the estimator $\widehat{\gamma}_{ps}$ of $\gamma_{ps}$. Figure~\ref{figure: gamma_20} shows in red, the true $\gamma_{ps}(R_t)$ as a function of the interpolation parameter $t$, and in green, the estimator $\widehat{\gamma}_{ps}$, for four sample sizes $n$, and $d=20$. This is essentially illustrative, as each point in these plots were obtained on one single experiment. We can still get some information from these plots: the estimation is poor for very small $n$ and good for large $n$, as expected. More importantly, the estimator $\widehat{\gamma}_{ps}$ is far more accurate for small values of $t$, that is, for large values of $\gamma_{ps}$, as predicted by Proposition~\ref{prp:gamma_ratio}. We observed these findings are consistent when we ran more experiments.

As mentioned earlier, the results in Figure~\ref{figure: gamma_20} illustrative, as each point is based on a single trajectory. In order to confirm the good performances of the estimator $\widehat{\gamma}_{ps}$ of~\citet{wolfer2024improved} we sampled $100$ trajectories in the case $n=1000$, $d=100$. We reported for each value of $t$ the MSE of the estimator $\widehat{\gamma}_{ps}$  over all these $100$ replications. The results are in Figure~\ref{fig:MSE}. They confirm that the estimator is accurate, and also that the estimation is more difficult for large values of $t$, that is, small values of $\gamma_{ps}$.
\begin{figure}[t]
\centering
\includegraphics[width=0.95\linewidth]{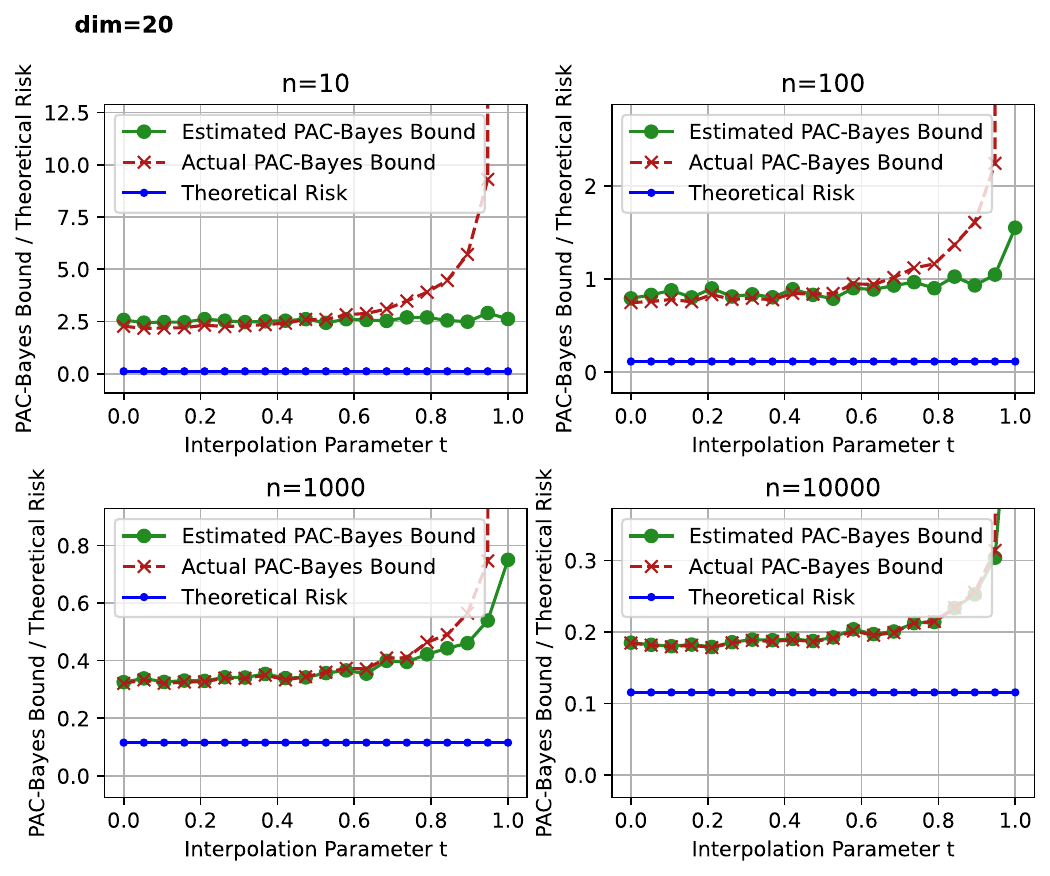}
\vspace*{14px}
\caption{Value of the PAC-Bayes bounds evaluated on a single trajectory, for $d=20$. In red, the non-empirical PAC-Bayes bound, as a function of $t$. In green, the empirical PAC-Bayes bound. In blue, the true value of the risk.}
\label{figure: pac-bayes}
\end{figure}

\textbf{Checking the PAC-bayes bounds:} for each value of $n$, $d$ and $t$, we then sampled a trajectory, and computed both the non-empirical and the empirical PAC-Bayes bounds on each trajectory. Figure~\ref{figure: pac-bayes} shows, for the various $n$, as functions of $t$: in red, the value of the non-empirical PAC-Bayes bound over all replications, in green, the value of the empirical PAC-Bayes bound over all replications, and in blue, the actual value of the risk $R(\hat{\theta}_{\text{ERM}})$. Here, we only show the results for $d=20$, the plots for the other values of $d$ are reported in the supplement. The take-home message is: for small sample size, the empirical bound is not a very good estimate of the non-empirical bound, but this is a regime where both bounds are vacuous anyway. For larger sample sizes, both bounds are non-vacuous and very similar. Finally, for very large $t$ (very small $\gamma_{ps}$), the non-empirical bound becomes unreliable: it seems to confirm that very mild assumptions such as $n^a\geq 1/\gamma_{ps}$ are indeed unavoidable in practice.

\section{CONCLUSION}

In this paper, we provided the first empirical PAC-Bayes bounds for Markov chains. The numerical results are encouraging: they show that, when the non-empirical bound is tight, the empirical bound is essentially as tight. This still relies on strong assumptions, and we believe that empirical bounds for time series beyond Markov chain are a very important research direction.

\begin{figure}[t]
    \centering
    \includegraphics[width = 0.95\linewidth]{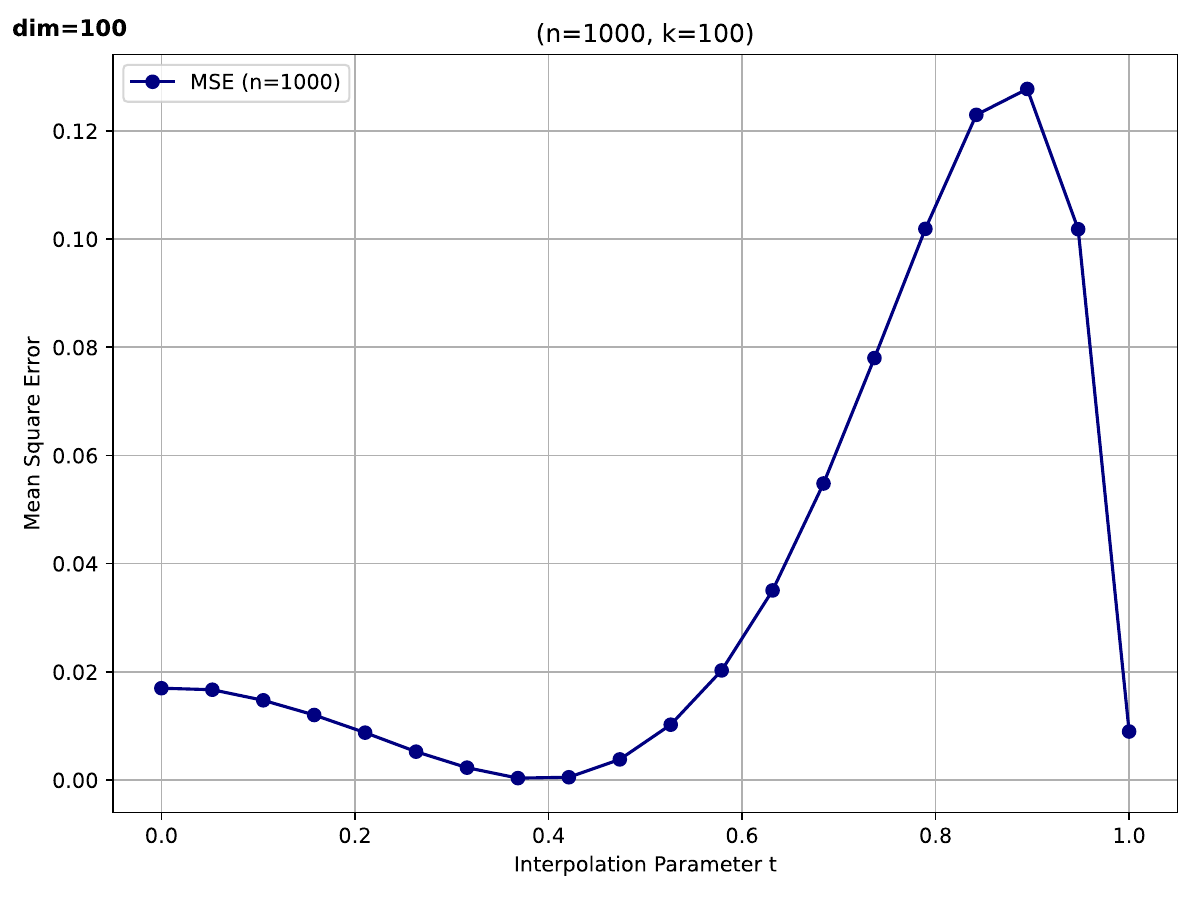}
    \caption{Mean Square Error of $k=100$ estimations of $\widehat\gamma_{ps}$.}
    \label{fig:MSE}
\end{figure}

\textbf{Content of the supplementary material:} all the proofs are in Appendix A. More details on the experiments, together with the results for $d\in\{4,10,50,100\}$ are provided in Appendix B. In Appendix C, we discuss how other PAC-Bayes bounds for time series from~\citet{PierreAlquier2013}, which rely on $\varphi$-mixing coefficients rather than the pseudo-spectral gap, could also be made empirical in the context of Markov chains.

\section*{Acknowledgments}

We thank Geoffrey Wolfer (Tokyo University of Agriculture and Technology) who introduced us to the pseudo-spectral gap and its estimation, and Mikołaj Kasprzak (ESSEC Business School) who spotted minor problems in the final version. We are grateful to the Anonymous Referees and the Area Chair for the very constructive feedback and the improvements they suggested. AI was used to track typos and formatting problems when preparing the camera-ready version, but it was not used at any other stage. All remaining mistakes are ours.

\newpage

\onecolumn
\aistatstitle{Supplementary Material}
\appendix

\section{PROOFS}

\subsection{A Preliminary Remark on the Observations}

In the introduction of the paper, we define the distribution of the pairs $(U_t,Y_t)$ by saying that the $(U_t)$ are sampled from a Markov chain with transition kernel $P$, and that the distribution of $Y_t$ given $(U_1,Y_1),\dots,(U_{t-1},Y_{t-1}),U_t$ is given by $Q(U_{t},\cdot)$. Observe that this is simply equivalent to stating that $[(U_t,Y_t)]_{t\geq 1}$ is a Markov chain on the space $\mathcal{U}\times\mathcal{Y}$ with transition kernel $\bar{P}$ given by
\begin{equation}
\label{fullkernel}
\bar{P}( (u,y),{\rm d}(u',y')) = P(u,{\rm d}u') Q(u',{\rm d} y').
\end{equation}
In the proof of Theorem~\ref{bern1}, we will actually use this fact. In particular, the pseudo-spectral gap of $\bar{P}$ will appear in the proof, while the assumption on Theorem~\ref{bern1} is on the pseudo-spectral gap $\gamma_{ps}$ of $P$: it turns out that this will not lead to any complication, as these two quantities are equal. We start by proving this fact.
\begin{lem}
\label{lemma:UY}
Assume $P$ is the transition kernel of an ergodic chain.
Let $\bar{P}$ be given by~\eqref{fullkernel}. The pseudo-spectral gap of $\bar{P}$ is equal to the pseudo-spectral gap $\gamma_{ps}$ of $P$.
\end{lem}
This might be well known among Markov chains specialists, but we did not find it in any classical textbook. Thus, we preferred to provide a complete proof.
\begin{proof}
The proof goes in four steps. In the first step, we will write the (cumbersome but) explicit formulas for $(\bar{P}^*)^k (\bar{P})^k$. In a second step, we prove that any $\tau$ in the spectrum of $(P^*)^k (P^k)$ is also in the spectrum of $(\bar{P}^*)^k (\bar{P})^k$. In the third step, we prove that any $\tau<1$ in the spectrum of $(\bar{P}^*)^k (\bar{P}^k)$ also belongs to the spectrum of $(P^*)^k (P^k)$. In the fourth and last step, we study the case of $\tau=1$, which is necessarily an eigenvalue for both operators: it appears that its multiplicity is $1$ for $(\bar{P}^*)^k (\bar{P})^k$ if and only if its multiplicity is also $1$ for $(P^*)^k (P^k)$.

\noindent \textbf{Step 1:} first, observe that as $P$ has a unique stationary distribution $\pi$, $\bar{P}$ has a unique stationary distribution $\bar{\pi}({\rm d}(u,y))=\pi({\rm d} u) Q(u,{\rm d}y) $. Thus, we have:
\begin{align*}
\bar{P}^* ( (u,y),{\rm d}(u',y'))
& = \frac{  \bar{P} ( (u',y'),{\rm d}(u,y)) \bar{\pi}({\rm d}(u',y')) }{ \bar{\pi}({\rm d}(u,y)) } 
\\
& = \frac{P(u',{\rm d}u) Q(u,{\rm d}y)   \pi({\rm d} u')   Q(u',{\rm d}y') }{ \pi({\rm d} u) Q(u,{\rm d}y) }
\\
& = \frac{P(u',{\rm d}u)   \pi({\rm d} u')   Q(u',{\rm d}y') }{ \pi({\rm d} u) }
\\
& = P^*(u,{\rm d}u') Q(u',{\rm d}y').
\end{align*}
Then, we use recursion to prove that $(\bar{P}^*)^k ( (u,y),{\rm d}(u',y')) = (P^*)^k(u,{\rm d}u') Q(u',{\rm d}y')$ and $\bar{P}^k ( (u,y),{\rm d}(u',y')) = P^k(u,{\rm d}u') Q(u',{\rm d}y')$. The proof is similar for both. For example,
\begin{align*}
\bar{P}^k ( (u,y),{\rm d}(u',y'))
& = \int_{(v,x)} \bar{P}^{k-1} ( (u,y),{\rm d}(v,x)) \bar{P}^{1} ( (v,x),{\rm d}(v,x)) 
\\
& = \int_{v} \int_x P^{k-1}(u,{\rm d}v) Q(v,{\rm d}x) P(v,{\rm d}u') Q(u',{\rm d}y')
\\
& = \int_{v}  P^{k-1}(u,{\rm d}v)  P(v,{\rm d}u') Q(u',{\rm d}y') \underbrace{\int_x Q(v,{\rm d}x)}_{=1}
\\
& = P^k(u,{\rm d}u') Q(u',{\rm d}y').
\end{align*}
We are now ready to work with $(\bar{P}^*)^k (\bar{P})^k $:
\begin{align*}
(\bar{P}^*)^k (\bar{P})^k ( (u,y),{\rm d}(u',y'))
& = \int_{(v,x)} (P^*)^k(u,{\rm d}v) Q(v,{\rm d}x) P^k(v,{\rm d}u') Q(u',{\rm d}y')
\\
& = \int_{v} (P^*)^k(u,{\rm d}v) P^k(v,{\rm d}u') Q(u',{\rm d}y') \int_x Q(v,{\rm d}x)
\\
& = (P^*)^k P^k (u,{\rm d}u') Q(u',{\rm d}y').
\end{align*}
The conclusion of the first step is thus:
\begin{equation}
\label{formula:pstar}
(\bar{P}^*)^k (\bar{P})^k ( (u,y),{\rm d}(u',y')) = (P^*)^k P^k (u,{\rm d}u') Q(u',{\rm d}y').
\end{equation}

\noindent \textbf{Step 2:} let $\tau \in {\rm sp}((P^*)^k P^k)$, that is, there are sequences of functions $(F_n(u))$ and $(G_n(u))$ with
$$ [(P^*)^k P^k - \tau I] F_n = G_n, $$
$$ \| F_n\|_{\pi}^2 = \int_u F_n(u)^2 \pi({\rm d}u) = 1  \text{ and } $$
$$ \| G_n\|_{\pi}^2 \xrightarrow[n\rightarrow \infty]{} 0 .$$
Then, put $ f_n(u,y) = F_n(u) $ and $g_n(u,y) = G_n(u) $, and observe that
\begin{align*}
[(\bar{P}^*)^k \bar{P}^k - \tau I] f_n (u,y)
&
= \int_{(u',y')} (\bar{P}^*)^k (\bar{P})^k ( (u,y),{\rm d}(u',y')) f_n(u',y') - \tau f_n(u,y)
\\
&
= \int_{(u',y')} (P^*)^k P^k (u,{\rm d}u') Q(u',{\rm d}y')  F_n(u') - \tau F_n(u)
\end{align*}
where we used both~\eqref{formula:pstar} and the definition of $f_n$: $f_n(u,y) = F_n(u)$. Thus,
\begin{align*}
[(\bar{P}^*)^k \bar{P}^k - \tau I] f_n
&
= \int_{u'} \int_{y'} (P^*)^k P^k (u,{\rm d}u') Q(u',{\rm d}y')  F_n(u') - \tau F_n(u)
\\
& = \int_{u'}  (P^*)^k P^k (u,{\rm d}u') F_n(u') \int_{y'} Q(u',{\rm d}y') -  \tau F_n(u)
\\
& =  [(P^*)^k P^k - \tau I] F_n(u)
\\
& = G_n (u)
\\
& = g_n(u,y).
\end{align*}
Moreover,
\begin{align*}
\| f_n\|_{\bar{\pi}}^2
& = \int_{(u,y)} f_n(u,y)^2 \bar{\pi}({\rm d}(u,y))
\\
& = \int_u \int_y F_n(u)^2 \pi({\rm d}u) Q(u, {\rm d}y)
\\
& = \int_u  F_n(u)^2 \pi({\rm d}u) \int_y Q(u, {\rm d}y)
\\
& = \| F_n\|_{\pi}^2
\\
& = 1
\end{align*}
and with exactly the same argument,
$$ \| g_n\|_{\bar{\pi}}^2  =  \| G_n\|_{\pi}^2 \xrightarrow[n\rightarrow \infty]{} 0 .$$
This proves that $\tau \in {\rm sp}((\bar{P}^*)^k \bar{P}^k)$.

\noindent \textbf{Step 3:} let us now assume $\tau \in {\rm sp}((\bar{P}^*)^k \bar{P}^k)$ and $\tau \neq 0$, that is, there are sequences $(f_n(u,y))$ and $(g_n(u,y))$ with
$$ [(\bar{P}^*)^k \bar{P}^k - \tau I] f_n = g_n, $$
$$ \| f_n\|_{\bar{\pi}}^2 = 1  \text{ and } $$
$$ \| g_n \|_{\bar{\pi}}^2 \xrightarrow[n\rightarrow \infty]{} 0 .$$
As
$$ (\bar{P}^*)^k \bar{P}^k f_n(u,y)
= \int_{(u',y')} (\bar{P}^*)^k \bar{P}^k ((u,y),{\rm d}(u',y')) f_n(u',y')
= \int_{(u',y')} (P^*)^k P^k (u,{\rm d}u') Q(u',{\rm d}y')  f_n(u',y')
$$
does not depend on $y$, the equality $ [(\bar{P}^*)^k \bar{P}^k - \tau I] f_n = g_n $ implies that $\tau f_n(u,y) + g_n(u,y) $ also does not depend on $y$, that is, we can for example write:
\begin{equation}
\label{weird:equation:0}
\tau f_n(u,y) + g_n(u,y)  = \int [\tau f_n(u,y') + g_n(u,y')] Q(u,{\rm d} y')
\end{equation}
and as a consequence
\begin{equation}
\label{weird:equation}
\tau f_n(u,y) - \tau \int f_n(u,y')  Q(u,{\rm d} y') + g_n(u,y) = \int g_n(u,y') Q(u,{\rm d} y') .
\end{equation}
Put
$$ F_n(u) = \int_{y'} f_n(u,y') Q(u,{\rm d}y') . $$
Then
\begin{align*}
[(P^*)^k P^k - \tau I] F_n (u)
&
= \int_{u'} (P^*)^k P^k (u, {\rm d}u') F_n(u') - \tau F_n(u)
\\
&
= \int_{u'} (P^*)^k P^k (u, {\rm d}u') \int_{y'} f_n(u',y') Q(u',{\rm d}y') - \tau \int_{y'} f_n(u,y') Q(u,{\rm d}y')
\\
&
= \int_{u'} \int_{y'} (P^*)^k P^k (u, {\rm d}u') Q(u',{\rm d}y') f_n(u',y') - \tau \int_{y'} f_n(u,y') Q(u,{\rm d}y')
\\
&
= \int_{(u',y')} (\bar{P}^*)^k \bar{P}^k ((u,y), {\rm d}(u',y')) f_n(u',y')
    - \tau \int_{y'} f_n(u,y') Q(u,{\rm d}y')
\end{align*}
where we used again~\eqref{formula:pstar}, and thus
\begin{align*}
[(P^*)^k P^k - \tau I] F_n (u)
& = (\bar{P}^*)^k \bar{P}^k f_n(u,y) - \tau \int_{y'} f_n(u,y') Q(u,{\rm d}y')
\\
& = [(\bar{P}^*)^k \bar{P}^k - \tau I] f_n(u,y) + \tau f_n(u,y) - \tau \int_{y'} f_n(u,y') Q(u,{\rm d}y')
\\
& = g_n(u,y) +  \tau f_n(u,y) - \tau \int_{y'} f_n(u,y') Q(u,{\rm d}y')
\\
& = \int g_n(u,y') Q(u,{\rm d} y') 
\end{align*}
according to~\eqref{weird:equation}. We can put $G_n(u) =  \int g_n(u,y') Q(u,{\rm d} y') $, we thus have $[(P^*)^k P^k - \tau I] F_n (u) = G_n(u)$. Using Jensen,
$$
\|G_n\|_{\pi}^2 = \int\left( \int g_n(u,y') Q(u,{\rm d} y')  \right)^2 \pi({\rm d}u)
\leq \int_u \int_{y'} g_n(u,y')^2 Q(u,{\rm d} y') \pi({\rm d}u) = \|g_n\|_{\bar{\pi}}^2 \rightarrow 0.
$$
Then, observe that
\begin{align*}
\| \tau F_n - G_n \|_{\pi}^2
&
= \int_{u} \left( \int_{y'} \tau f_n(u,y') Q(u,{\rm d} y') - \int_{y'} g_n(u,y') Q(u,{\rm d} y')  \right)^2 \pi({\rm d}u)
\\
& =  \int_{u} \int_{y}  \left( \int_{y'} \tau f_n(u,y') Q(u,{\rm d} y') - \int_{y'} g_n(u,y') Q(u,{\rm d} y')  \right)^2 \pi({\rm d}u) Q(u,{\rm d} y)
\\
& = \int_{u} \int_{y}  \left( \tau f_n(u,y) - g_n(u,y)   \right)^2 \pi({\rm d}u) Q(u,{\rm d} y)
\end{align*}
where we used~\eqref{weird:equation:0}, and thus
$$ \| \tau F_n - G_n \|_{\pi}^2
=
\| \tau f_n - g_n \|_{\bar{\pi}}^2,
$$
that is
$$ \| \tau F_n - G_n \|_{\pi}
=
\| \tau f_n - g_n \|_{\bar{\pi}}
$$
and finally
$$ \| \tau F_n \|_{\pi}
\geq \| \tau F_n - G_n \|_{\pi} - \| G_n \|_{\pi}
= \| \tau f_n - g_n \|_{\bar{\pi}} - \| G_n \|_{\pi} \xrightarrow[n\rightarrow \infty]{}
\tau.
$$
Besides, using Jensen again, we have
$$
\|F_n\|_{\pi}^2 = \int\left( \int f_n(u,y') Q(u,{\rm d} y')  \right)^2 \pi({\rm d}u)
\leq \int_u \int_{y'} f_n(u,y')^2 Q(u,{\rm d} y') \pi({\rm d}u) = \|f_n\|_{\bar{\pi}}^2 \rightarrow 1.
$$ 
Therefore, thanks to the sandwich theorem, and as $\tau\neq 0$, we can divide both sides by $\tau$ to get $ \| F_n \|_{\pi} \rightarrow 1$. So we found sequences $(F_n)$ and $(G_n)$ such that $[(P^*)^k P^k - \tau I] F_n (u) = G_ n(u)$, $ \| F_n \|_{\pi} \rightarrow 1$ and $ \| G_n \|_{\pi} \rightarrow 0$,
which proves that $\tau \in {\rm sp}((P^*)^k P^k) $.

\noindent \textbf{Step 4:} as $(P^*)^k P^k$ and $(\bar{P}^*)^k \bar{P}^k$ are both Markov kernels, they both admit $\tau=1$ as an eigenvalue, associated to the constant eigenfunction $F(u)=1$ for $(P^*)^k P^k$ and $f(u,y)=1$ for $(\bar{P}^*)^k \bar{P}^k$. In case $(P^*)^k P^k$ admits another eigenfunction $F_1(u)$ associated to $\tau=1$, it is necessarily non-constant, and we easily show that $f_1(u,y)=F_1(u)$ is then an eigenfunction of $(\bar{P}^*)^k \bar{P}^k$ associated to $\tau=1$ (computations similar to step 2), which is non-constant, and thus different from $f$. On the other hand, if  $f_1(u,y)$ is a non-constant eigenfunction $(\bar{P}^*)^k \bar{P}^k$ associated to $\tau=1$, the equation $(\bar{P}^*)^k \bar{P}^k f_1 = f_1 $ implies that $f_1(u,y)$ does not depend on $y$ (the left-hand side does not depend on $y$), which allows to define $F_1(u)=f_1(u,y)$. We then check immediately that $F_1(u)$ is an eigenfunction of $(P^*)^k P^k$ associated to $\tau=1$ and non constant, thus, different from $F(u)$.

\noindent \textbf{Conclusion of the proof:} from step 4, the eigenvalue $1$ has the same multiplicity for $(P^*)^k P^k$ and $(\bar{P}^*)^k \bar{P}^k$. If this multiplicity is larger than $1$, we have $
\gamma\left((\bar{P}^*)^k \bar{P}^k \right) 
= \gamma\left((P^*)^k P^k \right)=0$. Let us now assume that this multiplicity is equal to $1$.
Put ${\rm sp}_1((P^*)^k P^k) = \{\tau\in {\rm sp}((P^*)^k P^k): \tau\neq 1 \} \subset[0,1) $ and ${\rm sp}_1((\bar{P}^*)^k \bar{P}^k) = \{\tau\in {\rm sp}((\bar{P}^*)^k \bar{P}^k): \tau\neq 1 \}\subset[0,1)  $. From step 2, ${\rm sp}_1((P^*)^k P^k) \subset {\rm sp}_1((\bar{P}^*)^k \bar{P}^k)$. From  step 3, ${\rm sp}_1((\bar{P}^*)^k \bar{P}^k) \subset {\rm sp}_1((P^*)^k P^k)\cup\{0\}$. This proves that $\sup {\rm sp}_1((\bar{P}^*)^k \bar{P}^k) = \sup {\rm sp}_1((P^*)^k P^k)$. Thus
$$
\gamma\left((\bar{P}^*)^k \bar{P}^k \right) = 1- \sup {\rm sp}_1((\bar{P}^*)^k \bar{P}^k)
= 1-\sup {\rm sp}_1((P^*)^k P^k)
= \gamma\left((P^*)^k P^k \right).
$$
This concludes the proof.
\end{proof}

The variance of $\ell(f_\theta(U_t),Y_t)$ under the stationary distribution $\bar{\pi}$ will appear in some of the proofs. In order to keep formulas short enough, let us introduce a short notation for this quantity.
\begin{defi}
We put, for any $\theta\in\Theta$,
$$
V_{\ell(\theta)} :=
\int_{(u,y)} \left[ \ell(f_\theta(u),y) - R(\theta)\right]^2 \bar{\pi}({\rm d}(u,y))
.
$$
\end{defi}
\begin{remark}
In this paper, using the assumption $\ell \leq c$, we will always use the upper bound $V_{\ell(\theta)}\leq c^2 $. This bound can be poor in some situations. In the i.i.d. setting, important efforts were made to provide tighter empirical bounds for $V_{\ell(\theta)}$, we mentioned~\cite{pmlr-v26-seldin12a} (and many more papers) in the introduction. In this work, our primary objective was to provide empirical upper bounds on $\gamma_{ps}$, but we mention that providing tight empirical upper bounds on $V_{\ell(\theta)}$ in the Markov case would be extremely useful making our bounds tighter.
\end{remark}

\subsection{Proof of the Result in Section~\ref{section:nonempirical}}

\begin{proof}[Proof of Theorem~\ref{bern1}]
The theorem assumes that $(U_t)_{t}$ is a Markov chain with spectral gap $\gamma_{ps}>0$. Lemma~\ref{lemma:UY} ensures that the pairs $((U_t,Y_t))_{t}$ also form a Markov chain with the same spectral gap.
Using Theorem 3.4 of~\citet{paulin2015} (or, more precisely, the last inequality in the proof of Theorem 3.4), we have for every $s<\frac{\gamma_{ps}}{10}$
\begin{align*}\mathbb{E}_{\mathcal{S}}\Big[\exp \left(s n\left(R(\theta)-r(\theta)\right)\right) \Big] \leq  
\exp \left(\frac{2s^2\left(n+\frac{1}{\gamma_{ps}}\right)V_{\ell(\theta)}}{\gamma_{ps}-10s}\right).
\end{align*}
 By  plugging $s=\frac{\tilde{\lambda}}{n}$ for some $\tilde{\lambda}>0$, it becomes
\begin{align*}\mathbb{E}_{\mathcal{S}}\Big[\exp \left(\tilde{\lambda}\left(R(\theta)-r(\theta)\right)\right) \Big] \leq  
\exp \left(\frac{2\tilde{\lambda}^2 \left(n+\frac{1}{\gamma_{ps}}\right)V_{\ell(\theta)}}{n^2 (\gamma_{ps}-\frac{10\tilde{\lambda}}{n})}\right).
\end{align*}
For the convenience of writing put
$$A(\theta,\gamma_{ps},\tilde{\lambda})= 
\exp \left(\frac{2\tilde{\lambda}^2 \left(n+\frac{1}{\gamma_{ps}}\right)V_{\ell(\theta)}}{n^2 (\gamma_{ps}-\frac{10\tilde{\lambda}}{n})}\right).
$$
By rearranging terms and integrating both sides with respect to $\mu$ we get 
\begin{align*}\mathbb{E}_{\theta\sim\mu}\mathbb{E}_{\mathcal{S}}\Bigg[\exp \left(\tilde{\lambda}\left(R(\theta)-r(\theta)\right)-A(\theta,\gamma_{ps},\tilde{\lambda})\right) \Bigg] \leq 1.
\end{align*}
We can change the order of expectations due to Fubini–Tonelli's theorem, hence
\begin{align*}\mathbb{E}_{\mathcal{S}}\mathbb{E}_{\theta\sim\mu}\Bigg[\exp \left(\tilde{\lambda}\left(R(\theta)-r(\theta)\right)-A(\theta,\gamma_{ps},\tilde{\lambda})\right) \Bigg] \leq  1\;,
\end{align*}
then using Donsker and Varadhan's variational formula (see for example Lemma 2.2 page 28 of~\citet{Alquier_2024}) with $h(\theta)=\tilde{\lambda} R(\theta)- \tilde{\lambda} r(\theta)-A(\theta,\gamma_{ps},\tilde{\lambda})$ we arrive to
\begin{align*}\mathbb{E}_{\mathcal{S}}\Bigg[\exp  \Bigg(\sup_{ \rho\in\mathcal{P}(\Theta)} \biggl\{\mathbb{E}_{\theta\sim\rho}\Big[\tilde{\lambda}\left(R(\theta)-r(\theta)\right)-A(\theta,\gamma_{ps},\tilde{\lambda})\Big]+KL(\rho||\mu)\biggl\} \Bigg) \Bigg] \leq  1.
\end{align*}

Now let us transition to a probability bound, that is for any $s>0$
\begin{align*}&\mathop{\mathbb{P}}_{\mathcal{S}}\Bigg(\sup_{\rho\in\mathcal{P}(\Theta)}  \biggl\{\mathbb{E}_{\theta\sim\rho}\Big[\tilde{\lambda}\left(R(\theta) -r(\theta)\right)-A(\theta,\gamma_{ps},\tilde{\lambda})\Big]+KL(\rho||\mu)\biggl\} > s \Bigg)
 \\&\leq e^{-s}\mathbb{E}_{\mathcal{S}}\Bigg[\exp \Bigg(\sup_{\rho\in\mathcal{P}(\Theta)} \biggl\{\mathbb{E}_{\theta\sim\rho}\Big[\tilde{\lambda}\Big(R(\theta)-r(\theta)\Big)-A(\theta,\gamma_{ps},\tilde{\lambda})\Big]+ KL(\rho||\mu)\biggl\} \Bigg) \Bigg] \leq \exp(-s).
\end{align*}
By denoting $\delta=\exp(-s)$, and rewriting $A(\theta,\gamma_{ps},\tilde{\lambda})$ explicitly, we obtain
\begin{align*}
\mathbb{P}_{\mathcal{S}}\Bigg(\forall\rho\in\mathcal{P}(\Theta),\:\: \mathbb{E}_{\theta\sim\rho}\left[R(\theta)\right]\leq\mathbb{E}_{\theta\sim\rho}\left[r(\theta)\right]+\mathbb{E}_{\theta\sim\rho}\left[\ \frac{2\tilde{\lambda} \left(n+\frac{1}{\gamma_{ps}}\right)V_{\ell(\theta)}}{n^2 (\gamma_{ps}-\frac{10\tilde{\lambda}}{n})}\right]+ \frac{KL(\rho||\mu)+\log \frac{1}{\delta}}{\tilde{\lambda}}  \Bigg)\geq 1-\delta . 
\end{align*}
With a bounded loss $\ell(\cdot,\cdot)\leq c$, we are able to  bound the variance term $V_{\ell(\theta)}=\operatorname{var}_{\pi'}\left[\ell(f_\theta(U_t),Y_t)\right]\leq c^2$, and replace the expectation on right hand side by its upper bound:
\begin{align*}
\mathbb{P}_{\mathcal{S}}\Bigg(\forall\rho\in\mathcal{P}(\Theta),\: \mathbb{E}_{\theta\sim\rho}\left[R(\theta)\right]\leq\mathbb{E}_{\theta\sim\rho}\left[r(\theta)\right]+ \frac{2\tilde{\lambda} c^2\left(n+\frac{1}{\gamma_{ps}}\right)}{n^2 (\gamma_{ps}-\frac{10\tilde{\lambda}}{n})}+ \frac{KL(\rho||\mu)+\log \frac{1}{\delta}}{\tilde{\lambda}}  \Bigg)\geq1-\delta\;.
\end{align*}
Finally, put $\lambda = \tilde{\lambda} /\gamma_{ps} $ to obtain:
\begin{align*}
\mathbb{P}_{\mathcal{S}}\Bigg(\forall\rho\in\mathcal{P}(\Theta),\:\: \mathbb{E}_{\theta\sim\rho}\left[R(\theta)\right]\leq\mathbb{E}_{\theta\sim\rho}\left[r(\theta)\right]+ \frac{2\lambda c^2\left(1  + \frac{1}{n \gamma_{ps}} \right)}{n -10\lambda}+ \frac{KL(\rho||\mu)+\log \frac{1}{\delta}}{\lambda \gamma_{ps}}  \Bigg)\geq1-\delta\;\;.
\end{align*}

\end{proof}

\subsection{Proof of the Results in Section~\ref{section:empirical}}

\begin{proof}[Proof of Proposition ~\ref{prp:gamma_ratio}]

We start with Theorem 5.3 from~\cite{wolfer2022estimatingmixingtimeergodic}. The results are in the format of bounding the sample complexity $n$, and we restate them as concentration results. The last part of the proof of the theorem can be translated to the following concentration form:
    \begin{multline*}
    \mathbb{P}\Big(|\gamma_{ps}-\widehat{\gamma}_{ps}|\geq\varepsilon\Big)
    \\
    \leq \sum_{k=1}^{\ceil*{\frac{2}{\varepsilon}}}2d\exp \Bigg(-C \frac{n\varepsilon^2\pi_*k}{\|P\|_{\pi}\|P^k\|_1}\Bigg)
    +
    \sum_{k=1}^{\ceil*{\frac{2}{\varepsilon}}}\frac{d}{\sqrt{\pi_*}}\exp \Bigg(- \frac{n\varepsilon^2\pi_*\gamma_{ps}}{C}\Bigg)
    +
    \frac{Kd}{\sqrt{2\pi_*}}\exp \Bigg(- \frac{n\varepsilon^2\pi_*\gamma_{ps}}{C K^2}\Bigg)
\end{multline*}
We simplify it to the following confidence interval:
\begin{align*}
\mathbb{P}\Big(|\gamma_{ps}-\widehat{\gamma}_{ps}|\geq\varepsilon\Big)\leq\frac{C_{\text{ps}}d}{\varepsilon\sqrt{\pi_*}}\exp \Bigg(-n\varepsilon^2\pi_*\min\Big\{\gamma_{ps},\min_{1\leq k \leq {\ceil*{\frac{\varepsilon}{2}}}}\Big\{\frac{1}{\|P\|_{\pi}\|P^k\|_1}\Big\}\Big\}\Bigg)
\end{align*}
Or more concisely: 
\begin{align}\label{wolf and kont}
    \mathbb{P}\left(|\gamma_{ps}-\widehat{\gamma}_{ps}|\geq\varepsilon\right)\leq \frac{C_{\text{ps}}d}{\varepsilon\sqrt{\pi_*}} e ^ { -n\varepsilon^2\pi_*\min\left\{\gamma_{ps},\frac{1}{C(P)}\right\} }
\end{align}

Dividing both sides of the inequality by $\gamma_{ps}$ and re-defining $\varepsilon$ as $\varepsilon /\gamma_{ps}$, we obtain the result.
\end{proof}

\begin{proof}[Proof of Proposition~\ref{prp:ar1:1}]
First, observe that $(U_t)$ is a Gaussian process, so that the vector $(U_{t-1},U_t)$ is a Gaussian vector:
$$ (U_{t-1},U_t) \sim \mathcal{N}\left(\left(\begin{array}{c} 0 \\ 0 \end{array}\right),
  \left(\begin{array}{c c} \frac{1}{1-a^2} & \frac{a}{1-a^2} \\ \frac{a}{1-a^2} & \frac{1}{1-a^2} \end{array}\right)\right). $$
Thus, if we put $W_t = U_t\sqrt{1-a^2}$, it defines a Markov chain which is also a Gaussian process with
$$ (W_{t-1},W_t) \sim \mathcal{N}\left(\left(\begin{array}{c} 0 \\ 0 \end{array}\right),
  \left(\begin{array}{c c} 1 & a \\ a & 1 \end{array}\right)\right). $$
Let $p(w,w')$ denote the joint density of $(W_{t-1},W_t)$ and $p(w)$ denote the marginal density of $W_t$. The transition kernel $P_W$ of $W$ can be written through its density: $P_W(w,{\rm d}w') = [p(w,w')/p(w)]{\rm d}w' $. Note that the symmetry $p(w,w')=p(w',w)$ leads to $p(w)P_W(w,{\rm d}w') {\rm d}w = p(w')P_W(w',{\rm d}w) {\rm d}w'$, that is, $P_W^*=P_W$. In other words, from the diagonalization of $P_W$ we will directly obtain the diagonalization of $P_W^* P_W = P_W^2$, from which we will deduce the diagonalization of $P$.

We will now use Mehler's formula, in the form stated by~\cite{kibble1945extension}:
$$ p(w,w') = p(w) p(w') \sum_{n=0}^{\infty} \frac{a^n}{n!} {\rm He}_n(w) {\rm He}_n(w') $$
where $({\rm He}_n)$ are the Hermite polynomials satisfying:
$$ \int_\mathbb{R} {\rm He}_n(x) {\rm He}_m (x) \frac{ {\rm e}^{-\frac{x^2}{2}} }{\sqrt{2\pi}} {\rm d}x  = n! \delta_{nm}. $$
Plugging this in the formula for $P_W$, we obtain:
\begin{align*}
P_W(w,{\rm d}w')
& = \frac{p(w)p(w')\sum_{n=0}^{\infty} \frac{a^n}{n!} {\rm He}_n(w) {\rm He}_n(w')}{p(w)} {\rm d}w'
\\
& = p(w')\sum_{n=0}^{\infty} \frac{a^n}{n!} {\rm He}_n(w) {\rm He}_n(w') {\rm d}w'.
\end{align*}
This gives the diagonalization of $P_W$. The eigenfunctions of $P_W$ are the $({\rm He}_n)$ with corresponding eigenvalues $a^n$:
\begin{align*}
\int  {\rm He}_m(w') P_W(w,{\rm d}w') & = \int  {\rm He}_m(w') p(w')\sum_{n=0}^{\infty} \frac{a^n}{n!} {\rm He}_n(w) {\rm He}_n(w') {\rm d}w'
\\
& =\sum_{n=0}^{\infty} \frac{a^n}{n!} {\rm He}_n(w)  \int p(w') {\rm He}_n(w') {\rm He}_m(w') {\rm d}w'
\\
& =  \sum_{n=0}^{\infty} \frac{a^n}{n!} {\rm He}_n(w)  \int \frac{ {\rm e}^{-\frac{x^2}{2}} }{\sqrt{2\pi}}  {\rm He}_n(w') {\rm He}_m(w') {\rm d}w'
\\
& =\sum_{n=0}^{\infty} \frac{a^n}{n!} {\rm He}_n(w) n! \delta_{nm}
\\
& = a^m {\rm He}_m(w).
\end{align*}
The last thing to note is that we want the diagonalization of $P$, not of $P_W$.
However, observe that $P(u,{\rm d}u') = P_W( u/ \sqrt{1-a^2} , {\rm d}u'/\sqrt{1-a^2}  ) $ which obviously has different eigenfunctions ${\rm He}_n( \cdot / \sqrt{1-a^2}  )$ but the same eigvenvalues $a^n$:
\begin{align*}
\int {\rm He}_m(u'/\sqrt{1-a^2} ) P(u,{\rm d}u')
& = \int  {\rm He}_m(u'/\sqrt{1-a^2} ) P_W( u/ \sqrt{1-a^2} , {\rm d}u'/ \sqrt{1-a^2} ) 
\\
& = \int  {\rm He}_m(w') P_W( u/ \sqrt{1-a^2} , {\rm d}w'  ) \quad\quad (\text{ by c.o.v. } w' =  u'/\sqrt{1-a^2}) 
\\
& = a^m {\rm He}_m( u/ \sqrt{1-a^2} ).
\end{align*}
Thus, the eigenvalues of $P$ are $\{1,a,a^2,a^3,\dots\}$ and the eigenvalues of $P^* P $ are $\{1,a^2,a^4,a^6,\dots\}$, that is, $\gamma_{ps}=1-a^2$.
\end{proof}

\begin{proof}[Proof of Proposition~\ref{prp:ar1:2}]
We follow~\cite{nakakita2024dimension}: if we can prove that the sequence $(U_1,\dots,U_n)$ satisfies a log-Sobolev inequality with constant $K>0$, then, Theorem 1 in~\cite{nakakita2024dimension} gives, with probability at least $1-\exp(-t)$,
$$
\left| \frac{1}{n}\sum_{t=1}^n U_t^2 - \mathbb{E}(U_t^2) \right|
\leq
12\sqrt{K} |\mathbb{E}(U_t^2)| \sqrt{ \frac{9 +4t}{n}}.
$$
However, Section 3.2 in~\cite{nakakita2024dimension} shows that, if the sequence $(\zeta_t)$ satisfies a log-Sobolev inequality with constant $K_\zeta>0$, then the sequence $(U_1,\dots,U_n)$ satisfies a log-Sobolev inequality with constant
$$ K = K_\zeta \frac{1-a^2}{(1-|a|)^2}. $$
Moreover, Theorem 5.4 in~\cite{boucheron2006concentration}) actually states that the sequence $(\zeta_t)$ satisfies a log-Sobolev inequality with constant $K_\zeta=1$.
From the previous proposition,  $\mathbb{E}(U_t^2)=\frac{1}{1-a^2} = \frac{1}{\gamma_{ps}}$. Putting $\delta = \exp(-t)$, we obtain:
\begin{align*}
\left| \frac{1}{n}\sum_{t=1}^n U_t^2 - \frac{1}{\gamma_{ps}}\right|
\leq
12\sqrt{ \frac{1-a^2}{(1-|a|)^2} } \frac{1}{\gamma_{ps}}  \sqrt{ \frac{9 +4\log \frac{1}{\delta}}{n}}.
\end{align*}
Observe that $\gamma_{ps}= 1-a^2 \leq 1$ by definition. Thus,
\begin{align*}
\left| \frac{1}{\widehat{\gamma}_{ps}} - \frac{1}{\gamma_{ps}}\right|
=
\left| \frac{1}{\min\left(1, \left(\frac{1}{n}\sum_{t=1}^n U_t^2\right)^{-1}  \right)} - \frac{1}{\gamma_{ps}}\right|
\leq
\left| \frac{1}{n}\sum_{t=1}^n U_t^2- \frac{1}{\gamma_{ps}}\right|,
\end{align*}
and hence
$$
\left|  \frac{1}{\widehat{\gamma}_{ps}}  - \frac{1}{\gamma_{ps}}\right|
\leq
\frac{12}{(1-|a|)\sqrt{\gamma_{ps}}}  \sqrt{ \frac{9 +4\log \frac{1}{\delta}}{n}}.
$$
Multiply both sides by $\widehat{\gamma}_{ps} \leq 1$ to get
\begin{align*}
\left|  1 - \frac{\widehat{\gamma}_{ps}}{\gamma_{ps}}\right|
& \leq
\frac{12}{(1-\sqrt{1-\gamma_{ps}})\sqrt{\gamma_{ps}}}  \sqrt{ \frac{9 +4\log \frac{1}{\delta}}{n}} \leq \frac{24}{\gamma_{ps}^{3/2}}  \sqrt{ \frac{9 +4\log \frac{1}{\delta}}{n}}.
\end{align*}
\end{proof}

\begin{lem}\label{lem:reversed}
The statement of Theorem~\ref{bern1} remains valid when $r(\theta)$ and $R(\theta)$ are interchanged, that is:
\begin{align*}
\mathbb{P}_{\mathcal{S}}\left(\forall\rho\in\mathcal{P}(\Theta),\:\: \mathbb{E}_{\theta\sim\rho}\left[r(\theta)\right]\leq\mathbb{E}_{\theta\sim\rho}\left[R(\theta)\right]+ \frac{2\lambda c^2\left(1+\frac{1}{n \gamma_{ps}}\right)}{n - 10\lambda}+ \frac{KL(\rho||\mu)+\log \frac{1}{\delta}}{\lambda \gamma_{ps}}  \right)\geq1-\delta.
\end{align*}

\end{lem}
\begin{proof}[Proof of Lemma~\ref{lem:reversed}]
The inequality in Theorem~\ref{bern1} is originally stated with the following form
\[
\mathbb{E}_{\theta \sim \rho}[R(\theta)] \leq \mathbb{E}_{\theta \sim \rho}[r(\theta)] + \frac{2\lambda c^2\left(1+\frac{1}{n \gamma_{ps}}\right)}{n - 10\lambda}+ \frac{KL(\rho||\mu)+\log \frac{1}{\delta}}{\lambda \gamma_{ps}}.
\]
We observe that replacing the loss function \(\ell(\cdot,\cdot)\) by its negative, i.e., defining \(\tilde{\ell} := -\ell\), allows us to reverse the inequality. Since all derivations in the proof of Theorem~\ref{bern1} depend linearly on \(\ell\), the same steps apply with \(\tilde{\ell}\), leading to
\[
\mathbb{E}_{\theta \sim \rho}[-R(\theta)] \leq \mathbb{E}_{\theta \sim \rho}[-r(\theta)] + \frac{2\lambda c^2\left(1+\frac{1}{n \gamma_{ps}}\right)}{n - 10\lambda}+ \frac{KL(\rho||\mu)+\log \frac{1}{\delta}}{\lambda \gamma_{ps}},
\]
which results in the claimed bound:
\[
\mathbb{E}_{\theta \sim \rho}[r(\theta)] \leq \mathbb{E}_{\theta \sim \rho}[R(\theta)] + \frac{2\lambda c^2\left(1+\frac{1}{n \gamma_{ps}}\right)}{n - 10\lambda}+ \frac{KL(\rho||\mu)+\log \frac{1}{\delta}}{\lambda \gamma_{ps}}.
\]
\end{proof}

It is straightforward to notice that all other developments after Theorem 2.1 can be done in analogous manner also for the interchanged version of the bound. 

\begin{proof}[Proof of Theorem~\ref{excess}]
We rewrite Proposition~\ref{prp:gamma_ratio} in two ways:
\begin{align}\label{aux estimation1}
     \gamma_{ps}-\widehat{\gamma}_{ps}&\leq\left|\gamma_{ps}-\widehat{\gamma}_{ps}\right|\leq\left|\frac{\widehat{\gamma}_{ps}-\gamma_{ps}}{\widehat{\gamma}_{ps}}\right|\leq\left|1-\frac{\gamma_{ps}}{\widehat{\gamma}_{ps}}\right|\leq\varepsilon\quad\quad\text{[using that $\widehat{\gamma}_{ps}\leq1$]}\\\frac{1}{\gamma_{ps}}&\leq\frac{1+\varepsilon}{\widehat{\gamma}_{ps}}.
     \label{aux estimation2}
\end{align}
Both of them simultaneously hold with probability $1-\alpha(n,\gamma_{ps},\varepsilon)$.

Recall that
\begin{align*}
B\left(\nu,u\right) = \min_{\lambda\in\Lambda} \Bigg\{\frac{2\lambda c^2\left(1+\frac{u}{n} \right)}{n - 10\lambda} + u \frac{KL(\nu||\mu)+\log \frac{L}{\delta}}{\lambda}\Bigg\} 
\end{align*}
and
$$\hat{\rho}=\argmin_{\rho}\left[\mathbb{E}_{\theta\sim\rho}\big[r(\theta)\big]+ B\left(\rho,\frac{(1+\varepsilon)}{\widehat{\gamma}_{ps}}\right)\right].$$
The function $B$ is a non-decreasing function on a second variable, hence applying \eqref{aux estimation2} to the bound in Theorem~\ref{bern1}, by union bound argument we have with probability $1-\delta-\alpha$ for all $\rho\in\mathcal{P}(\Theta)$
\begin{align*}
    \mathbb{E}_{\theta\sim\rho}[R(\theta)]\leq\mathbb{E}_{\theta\sim\rho}[r(\theta)]+B\left(\rho,\frac{1+\varepsilon}{\widehat{\gamma}_{ps}}\right).
\end{align*}
Particularly
\begin{alignat*}{2}
    \mathbb{E}_{\theta\sim\hat{\rho}}[R(\theta)]&\leq\mathbb{E}_{\theta\sim\hat{\rho}}[r(\theta)]+B\left(\hat{\rho},\frac{1+\varepsilon}{\widehat{\gamma}_{ps}}\right)\hspace{2cm}&&[\text{w.p. $1-\delta-\alpha$}] \\&\leq\mathbb{E}_{\theta\sim\rho}[r(\theta)]+B\left(\rho,\frac{1+\varepsilon}{\widehat{\gamma}_{ps}}\right)&& [\text{for all}\; \rho,\text{by definition of}\;\hat{\rho}]
    \\&\leq\mathbb{E}_{\theta\sim\rho}[R(\theta)]+B\left(\rho,\frac{1}{\gamma_{ps}}\right)+B\left(\rho,\frac{1+\varepsilon}{\widehat{\gamma}_{ps}}\right) \quad\quad\quad &&[\text{w.p. $1-2\delta-\alpha$ by Lemma~\ref{lem:reversed}}]
    \\&\leq\mathbb{E}_{\theta\sim\rho}[R(\theta)]+B\left(\rho,\frac{1}{\gamma_{ps}}\right)+B\left(\rho,\frac{1+\varepsilon}{\gamma_{ps}-\varepsilon}\right)&&[\text{by~\eqref{aux estimation1}}]
    \\&\leq\mathbb{E}_{\theta\sim\rho}[R(\theta)]+2B\left(\rho,\frac{1+\varepsilon}{\gamma_{ps}-\varepsilon}\right).
\end{alignat*}
Thus with probability at least $1-2\delta-\alpha(n,\gamma_{ps},\varepsilon)$, we have 
\begin{align*}
    \mathbb{E}_{\theta\sim\hat{\rho}}[R(\theta)]
    \leq\inf_{\substack{\rho\in\mathcal{P}(\Theta) \\ \lambda\in\Lambda}}\Bigg[\mathbb{E}_{\theta\sim\rho}[R(\theta)]+2B\left(\rho,\frac{1+\varepsilon}{\gamma_{ps}-\varepsilon}\right)  \Bigg].
\end{align*}
\end{proof}
\begin{remark}\label{re: exact B}
Instead of relaxing the last line of the proof, it is also possible to leave it in the exact form by rewriting it as 
\[B\left(\rho,\frac{1}{\gamma_{ps}}\right)+B\left(\rho,\frac{1+\varepsilon}{\gamma_{ps}-\varepsilon}\right)=B\left(\rho,\frac{1}{\gamma_{ps}}+\frac{1+\varepsilon}{\gamma_{ps}-\varepsilon}\right)= B\left(\rho,\frac{2\gamma_{ps}+\varepsilon(1-\gamma_{ps})}{\gamma_{ps}(\gamma_{ps}-\varepsilon)}\right). \]
Thus it will result in the slightly tighter, but maybe less readable:
\begin{align*}
    \mathbb{E}_{\theta\sim\hat{\rho}}[R(\theta)]
    \leq\inf_{\substack{\forall\rho\in\mathcal{P} \\ \lambda\in\Lambda}}\Bigg[\mathbb{E}_{\theta\sim\rho}[R(\theta)]+2B\left(\rho,\frac{2\gamma_{ps}+\varepsilon(1-\gamma_{ps})}{\gamma_{ps}(\gamma_{ps}-\varepsilon)}\right)  \Bigg].
\end{align*}
\end{remark}

\subsection{Proof of the Results in Section~\ref{section:application}}

\begin{proof}[Proof of Theorem~\ref{thm: finite_erm theoretical}]
We start by an application of Theorem~\ref{bern1}, with probability at least $1-\delta$ on the sample $\mathcal{S}$,
$$
\forall \rho\in\mathcal{P}(\Theta)\text{, }
\mathbb{E}_{\theta\sim\rho}\big[R(\theta)\big]\leq\mathbb{E}_{\theta\sim\rho}\big[r(\theta)\big]+ B(\rho,\gamma_{ps},\lambda) .$$
In particular, this holds for any $\rho$ in the set of Dirac masses $\{\delta_{\theta} \:|\: \theta\in\Theta\}$. Thus, we have, with probability at least $1-\delta$,
$$\forall \theta\in\Theta\text{, }\mathbb{E}_{\theta\sim\delta_\theta}\big[R(\theta)\big]\leq\mathbb{E}_{\theta\sim\delta_\theta}\big[r(\theta)\big] + B(\delta_{\theta},\gamma_{ps},\lambda)$$ 
Having  $\mathbb{E}_{\theta\sim\delta_\theta}\big[R(\theta)\big]=R(\theta)$, $\mathbb{E}_{\theta\sim\delta_\theta}\big[r(\theta)\big]=r(\theta)$, and
$$KL(\rho||\mu)=\sum_{\vartheta}\log\left(\frac{\delta_\theta(\vartheta)}{\mu(\vartheta)}\right)\delta_{\theta}(\vartheta)=\log\frac{1}{\mu(\theta)}$$
we get that for any $\theta\in\Theta$
$$R(\theta)\leq r(\theta) +             \frac{2\lambda c^2\left(1+\frac{1}{n \gamma_{ps}}\right)}{n - 10\lambda}+ \frac{\log \frac{1}{\mu(\theta)\delta}}{\lambda \gamma_{ps} }$$ 
with probability $1-\delta$.\\
In particular, by putting $\theta=\hat{\theta}_{\text{ERM}}$ we obtain
$$R(\hat{\theta}_{\text{ERM}})\leq \min_{\theta}r(\theta) +             \frac{2\lambda c^2\left(1+\frac{1}{n \gamma_{ps}}\right)}{n - 10\lambda}+ \frac{\log \frac{1}{\mu(\theta)\delta}}{\lambda \gamma_{ps} } $$
With the assumption that $\mu$ is uniform 
\begin{align}R(\hat{\theta}_{\text{ERM}})\leq \min_{\theta}r(\theta) +          \frac{2\lambda c^2\left(1+\frac{1}{n \gamma_{ps}}\right)}{n - 10\lambda}+ \frac{\log \frac{M}{\delta}}{\lambda \gamma_{ps} }.
\end{align}
Assume that $10\lambda< n\varepsilon/(1+\varepsilon)$, then:
\begin{align}
\label{with 10lambda and n} R(\hat{\theta}_{\text{ERM}})\leq \min_{\theta}r(\theta) +          \frac{2(1+\varepsilon)\lambda c^2\left(1+\frac{1}{n \gamma_{ps}}\right)}{n}+ \frac{\log \frac{M}{\delta}}{\lambda \gamma_{ps} }.
\end{align}
Then we minimize the right-hand side by choosing the optimal $\lambda$. It is achieved for 
$$\lambda_{op}=\sqrt{\frac{  n\log \frac{M}{\delta}}{2(1+\varepsilon) c^2\gamma_{ps}\left(1+\frac{1}{\gamma_{ps}n}\right)}}$$
in which case the bound settles into its final form:
\begin{align*}R(\hat{\theta}_{\text{ERM}})\leq \min_{\theta}r(\theta) +           \sqrt{\frac{  8(1+\varepsilon) c^2 \log \frac{M}{\delta} }{ \gamma_{ps}n }\left(1+\frac{1}{\gamma_{ps}n}\right)}.
\end{align*}
Note that our choice of $\lambda$ is only compatible with $ 10\lambda< n\varepsilon/(1+\varepsilon)$ when:
$$
n >   \frac{ 50(1+\varepsilon) \gamma_{ps}\log \frac{M}{\delta}}{\varepsilon^2 c^2\gamma_{ps}\left(1+\frac{1}{\gamma_{ps}n}\right)}.
$$
\end{proof}

\section{ADDITIONAL DETAILS ON THE EXPERIMENTS}

We described in the paper the construction of the transition matrices in our experiments. We remind that $
R_t := t P + (1-t) Q
$
where $P$ satisfies $\gamma_{ps}(P)\simeq 0$, and $Q$ is such that $\gamma_{ps}(Q) = 1$. We actually took $Q=1^T\cdot\pi$ where $1^T=(1,\dots,1)$ and $\pi$ is the invariant distribution of $P$. That is, a Markov chain whose transition kernel $Q$ is simply an i.i.d. sequence from $\pi$. The fact that both $P$ and $Q$ have the same invariant distribution $\pi$ ensures that the invariant distribution of $R_t$ is also
$\pi$, indeed:
$$
\pi R_t = \pi [ t P + (1-t)Q]  = t \pi P + (1-t) \pi Q = t \pi + (1-t)\pi = \pi.
$$
All this was detailed in the main body of the paper, but it remains to give the definition of $P$.

Our choice for $d=4$ is
\[
P = 
\begin{bmatrix}
\frac{1}{4} & \frac{1}{4} & \frac{1}{4} & \frac{1}{4} \\[2mm]
\frac{1}{4} & \frac{1}{4} & \frac{1}{4} & \frac{1}{4} \\[1mm]
p & 0 & 1-p & 0 \\[1mm]
0 & q & 0 & 1-q
\end{bmatrix}.
\]
We have fixed parameters $p=0.01$ and $q=0.001$. The heuristic reason behind such parameters is that, when the chain reaches state $3$ and $4$, it will stay stuck in this state for a very long time, making the convergence to the stationary distribution very slow. And indeed, we observed that $\gamma_{ps}$ is very close to $0$ for this chain.

For larger $d$, we generalized the construction in the following way:
\[
P =
\begin{bmatrix}
1/d & 1/d & 1/d & 1/d & \cdots & 1/d \\
1/d & 1/d & 1/d & 1/d & \cdots & 1/d \\
p            & 0            & 1-p          & 0            & \cdots & 0            \\
0            & q            & 0            & 1-q          & \cdots & 0            \\
1/d & 1/d & 1/d & 1/d & \cdots & 1/d \\
\vdots       & \vdots       & \vdots       & \vdots       & \ddots & \vdots       \\
1/d & 1/d & 1/d & 1/d & \cdots & 1/d 
\end{bmatrix}.
\]

In the main body of the paper, we assessed the estimator $\widehat{\gamma}_{ps}$ when $d=20$ in Figure~\ref{figure: gamma_20}, and the accuracy of the PAC-Bayes bound, also with $d=20$, in Figure~\ref{figure: pac-bayes}.

We now provide similar results when $d=4$ (Figures~\ref{figure: gamma_4} and~\ref{figure:bound_4} respectively), $d=10$ (Figures~\ref{figure: gamma_10} and~\ref{figure:bound_10}),  $d=50$ (Figures~\ref{figure: gamma_50} and~\ref{figure:bound_50}) and finally $d=100$ (Figures~\ref{figure: gamma_100} and~\ref{figure:bound_100}). The results remain essentially unchanged, note however that the estimation of $\gamma_{ps}$ becomes more challenging when $d$ is very large.

\begin{figure}[h]
    \centering
    \begin{minipage}{0.50\linewidth}
        \centering
        \includegraphics[width=1\linewidth]{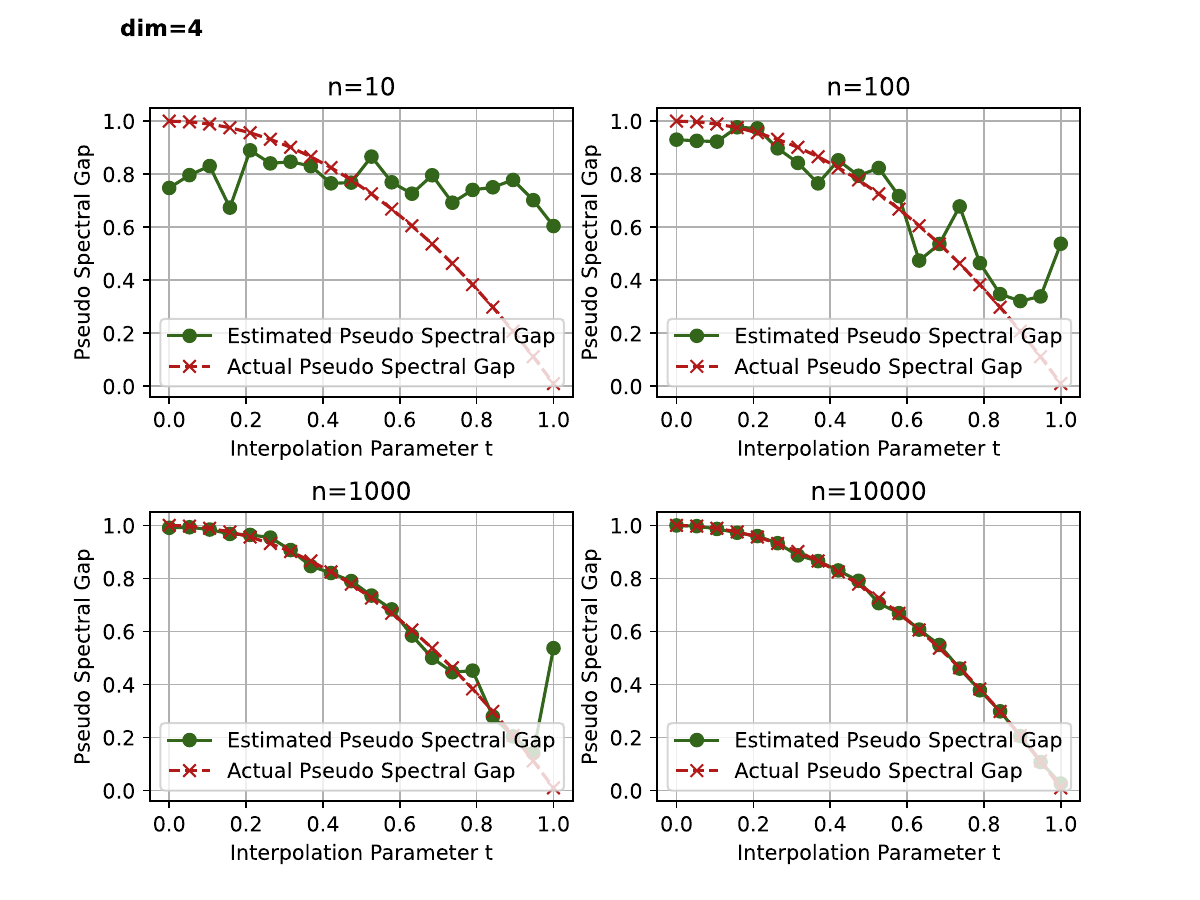}
    \caption{Estimation of $\gamma_{\text{\text{ps}}}$ when $d=4$.}
    \label{figure: gamma_4}
    \end{minipage}%
    \hfill
    \begin{minipage}{0.50\linewidth}
        \centering
        \includegraphics[width=1\linewidth]{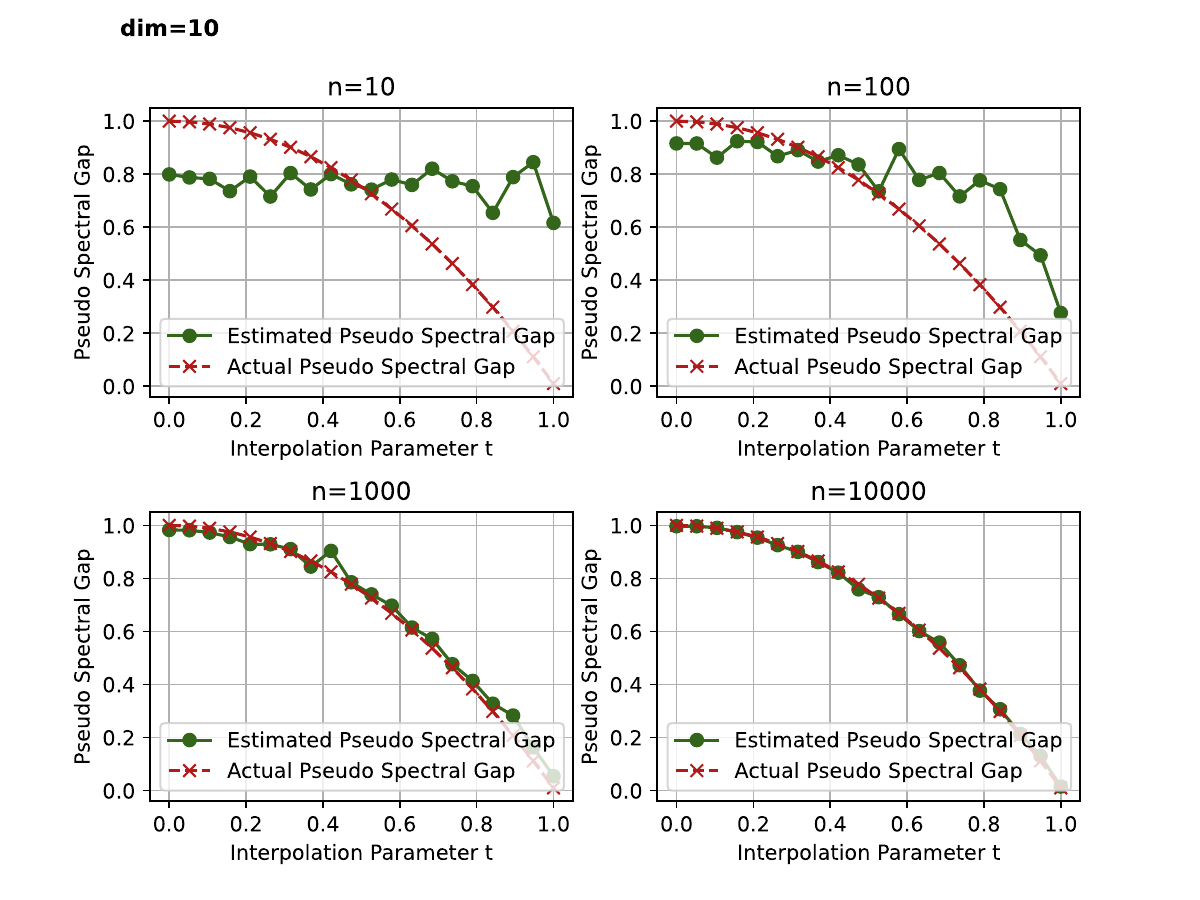}
    \caption{Estimation of $\gamma_{\text{\text{ps}}}$ when $d=10$.}
    \label{figure: gamma_10}
    \end{minipage}%
\end{figure}

\begin{figure}[h]
    \centering
    \begin{minipage}{0.50\linewidth}
        \centering
        \includegraphics[width=0.95\linewidth]{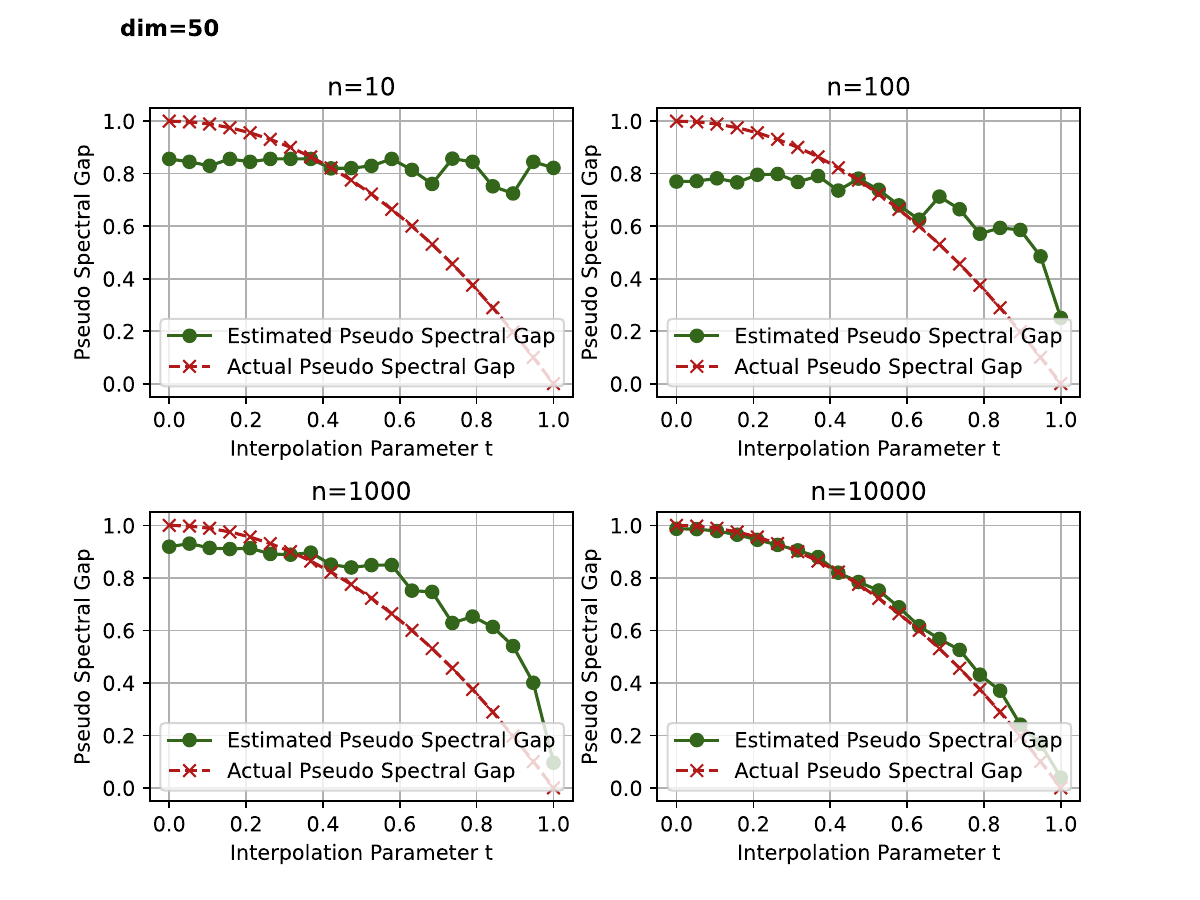}
    \caption{Estimation of $\gamma_{\text{\text{ps}}}$ when $d=50$.}
    \label{figure: gamma_50}
    \end{minipage}%
    \hfill
    \begin{minipage}{0.50\linewidth}
        \centering
        \includegraphics[width=0.95\linewidth]{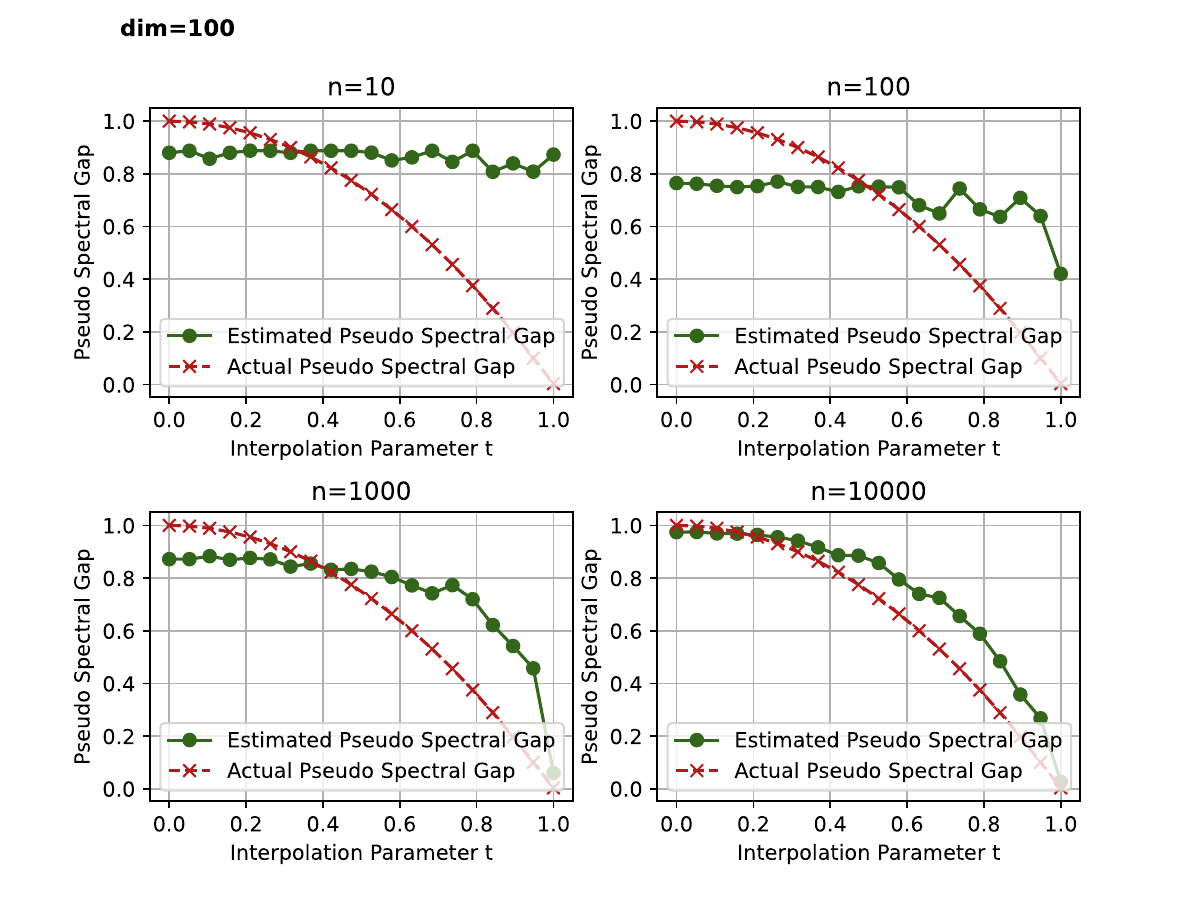}
    \caption{Estimation of $\gamma_{\text{\text{ps}}}$ when $d=100$.}
    \label{figure: gamma_100}
    \end{minipage}%
\end{figure}

\newpage
\vspace*{50px}

\begin{figure}[ht]
    \centering
    \begin{minipage}{0.50\linewidth}
        \centering
        \includegraphics[width=\linewidth]{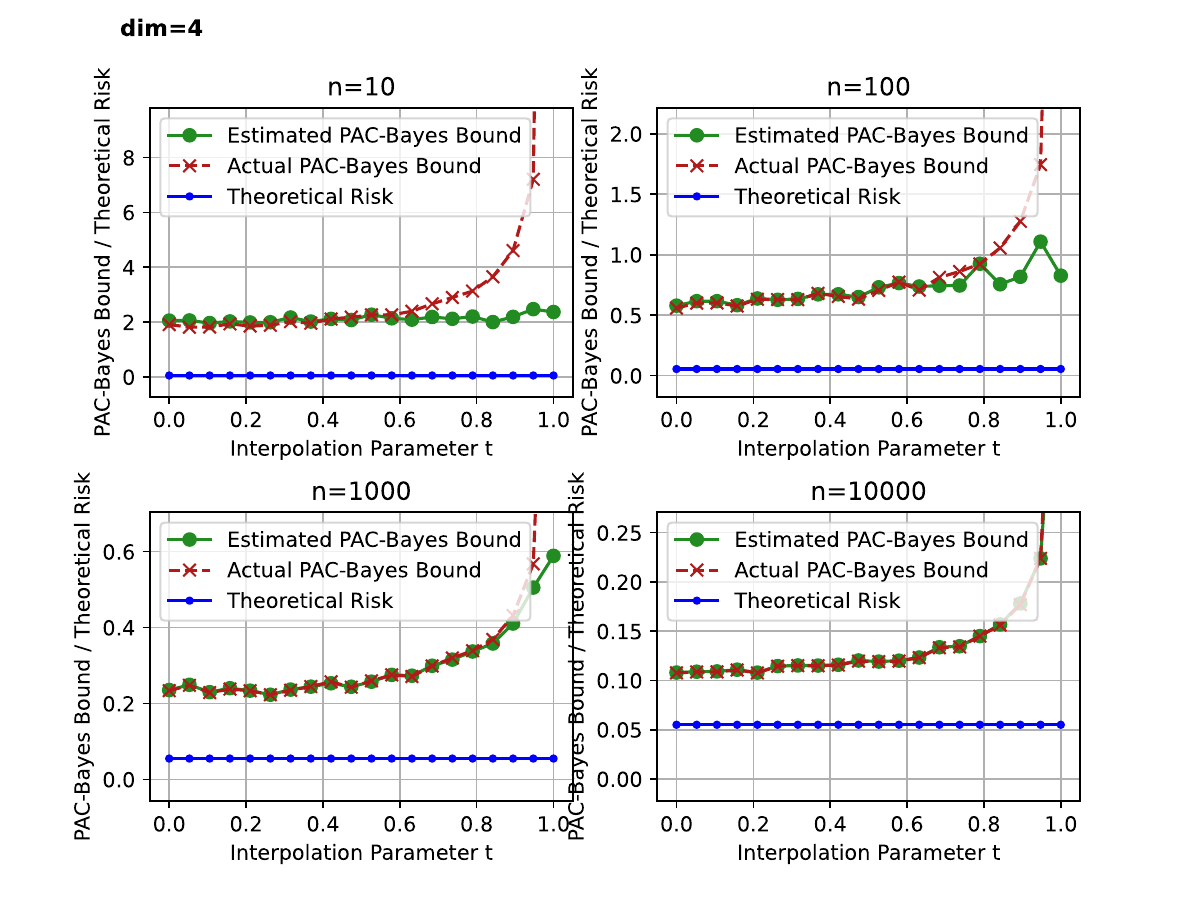}
        \caption{PAC-Bayes bounds \\for $R(\hat{\theta}_{\text{ERM}})$ when $d=4$.}
        \label{figure:bound_4}
    \end{minipage}%
    \hfill
    \begin{minipage}{0.50\linewidth}
        \centering
        \includegraphics[width=\linewidth]{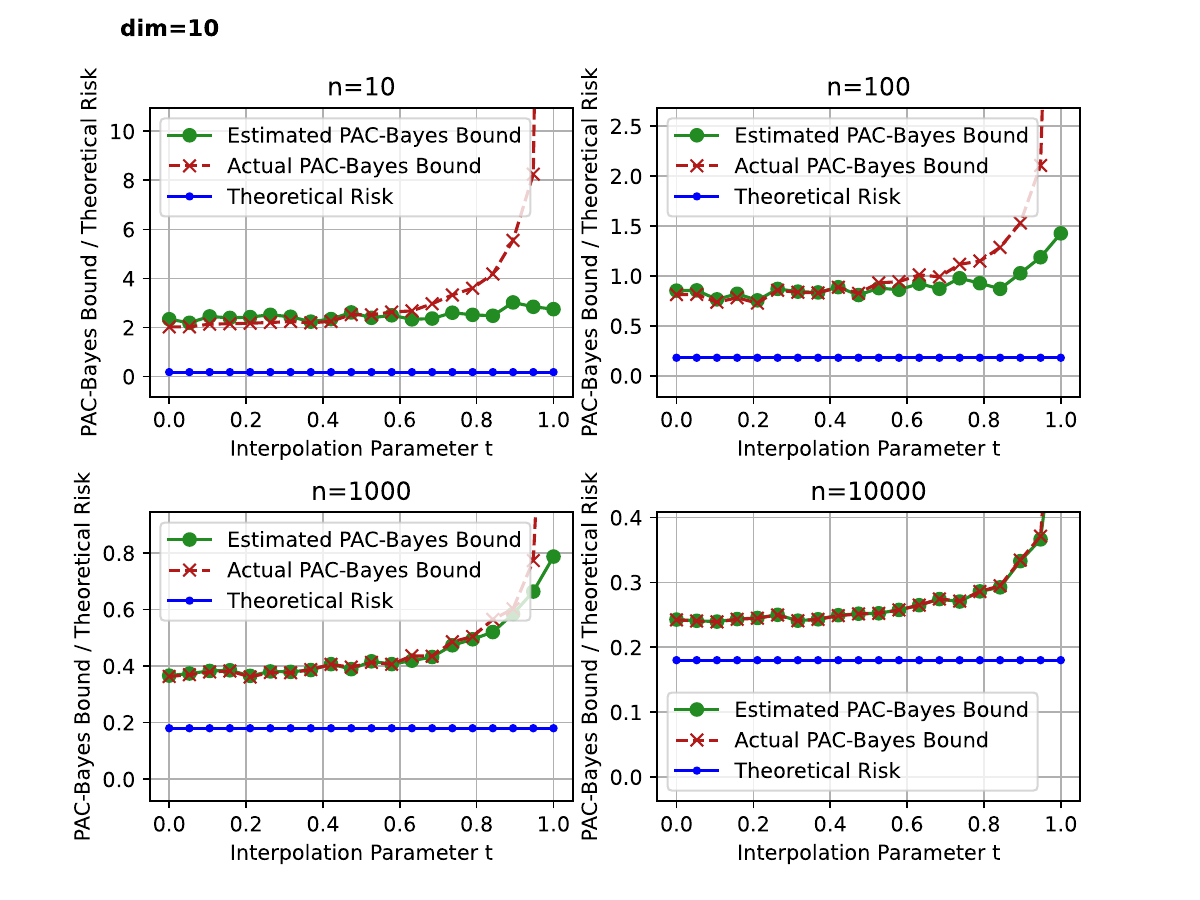}
        \caption{PAC-Bayes bounds \\for $R(\hat{\theta}_{\text{ERM}})$ when $d=10$.}
        \label{figure:bound_10}
    \end{minipage}
\end{figure}

\begin{figure}[ht]
    \centering
    \begin{minipage}{0.50\linewidth}
        \centering
        \includegraphics[width=\linewidth]{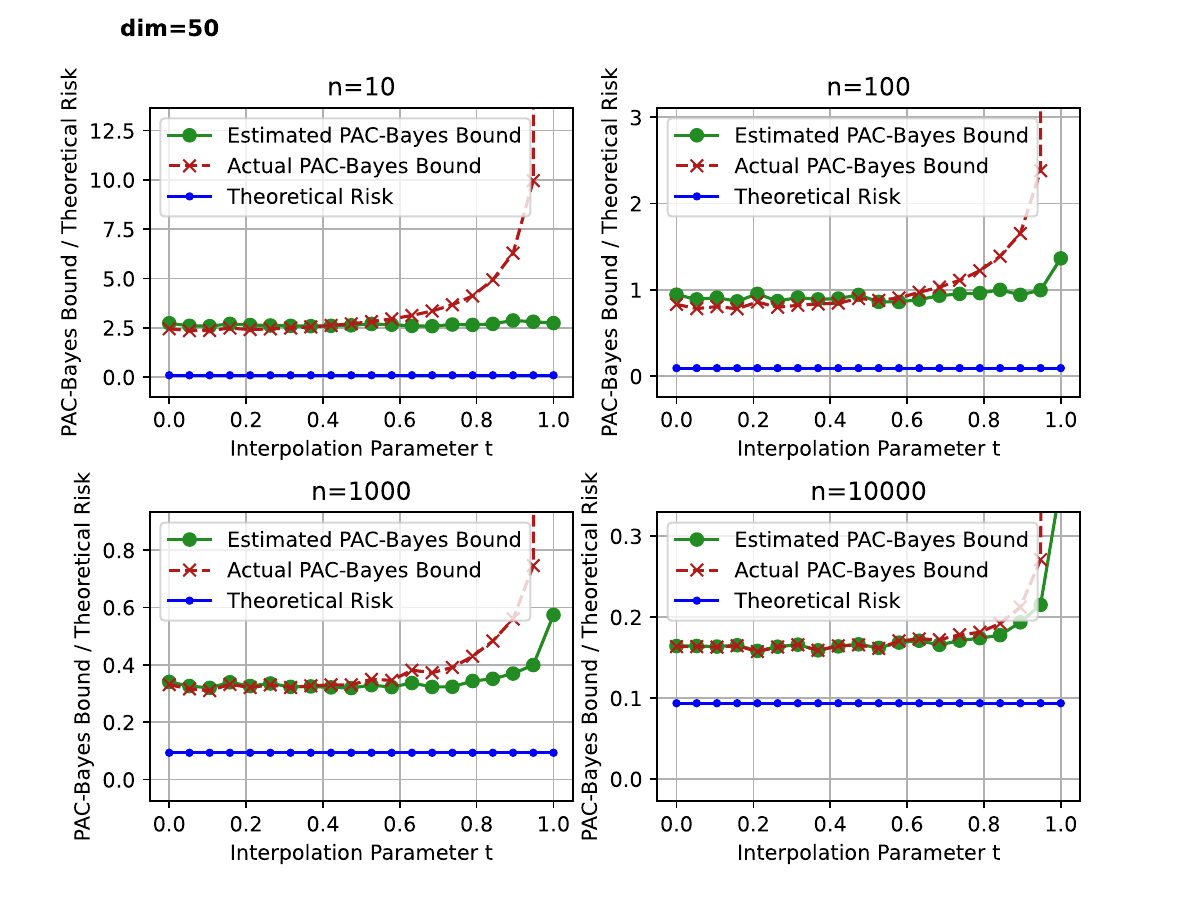}
        \caption{PAC-Bayes bounds \\for $R(\hat{\theta}_{\text{ERM}})$ when $d=50$.}
        \label{figure:bound_50}
    \end{minipage}%
    \hfill
    \begin{minipage}{0.50\linewidth}
        \centering
        \includegraphics[width=\linewidth]{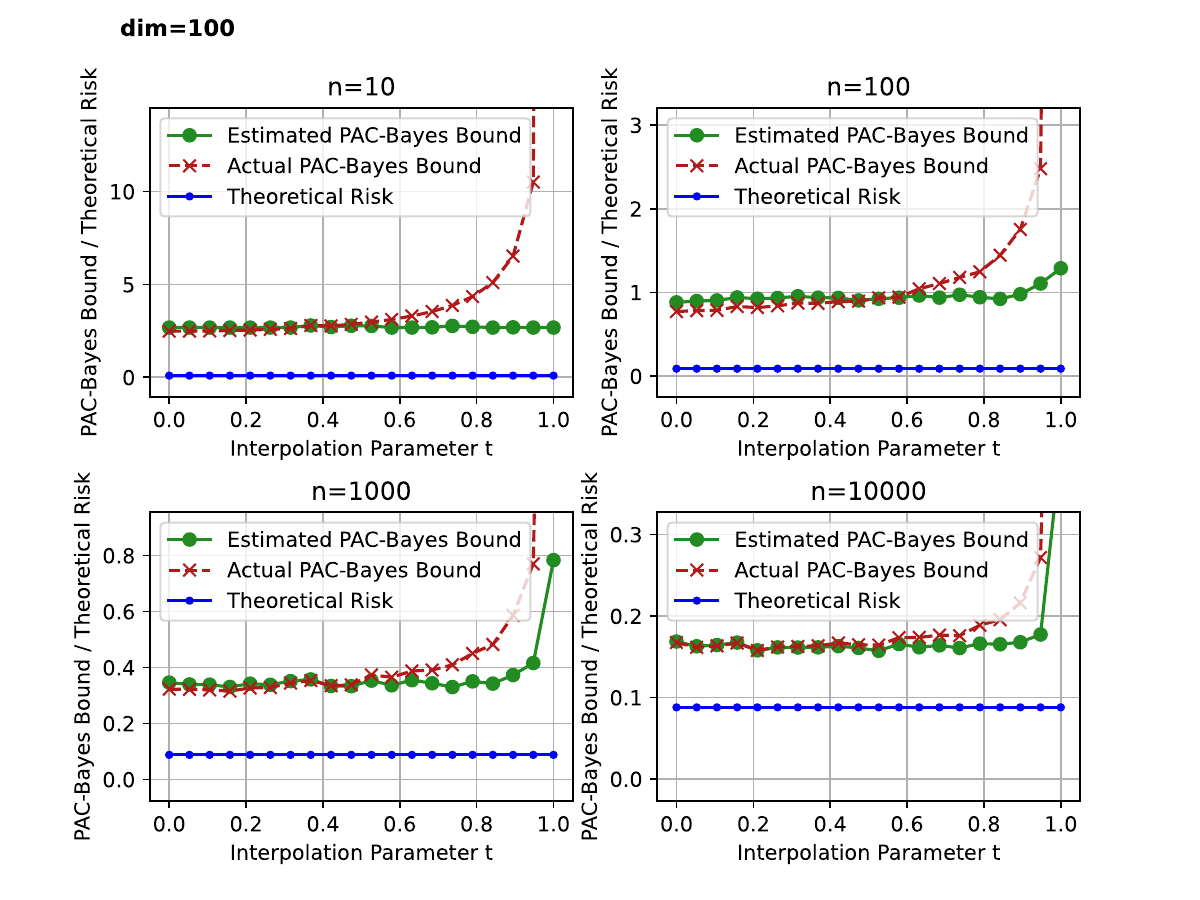}
        \caption{PAC-Bayes bounds for $R(\hat{\theta}_{\text{ERM}})$\\ $d=100$.}
        \label{figure:bound_100}
    \end{minipage}
\end{figure}

\section{PAC-BAYES FOR TIME SERIES WITH $\varphi$-MIXING}

In this final section, we discuss another PAC-Bayes bounds for time series due to~\citet{PierreAlquier2013}. This bound holds under an assumption on the $\varphi$-mixing coefficients of the series, that do not require the series to be a Markov chain. This bound is not empirical. We show that, under the additional assumption that the series is actually a Markov chain, we can upper-bound the $\varphi$-mixing coefficients by a function of the pseudo-spectral gap, and thus make the bound empirical.

\subsection{Hoeffding-Type PAC-Bayes Bound for Time Series}

Our starting point follows the framework introduced by \citet{rio2000}. Consider a sequence of metric spaces $(E_t, d_t)$ for $t = 1, \ldots, n$, each with diameter denoted by $\Delta_t$. Let $E^n := E_1 \times E_2 \times \cdots \times E_n$. A real-valued function $f : E^n \to \mathbb{R}$ is said to be $M$-Lipschitz if for all $(x_1, \ldots, x_n), (y_1, \ldots, y_n) \in E^n$, we have
\[
\big|f(x_1, \ldots, x_n) - f(y_1, \ldots, y_n)\big| \leq M \sum_{t=1}^{n} d_t(x_t, y_t).
\]
Now, let $(X_1, \ldots, X_n)$ be a sequence of random variables, and for each $t \in \{1, \ldots, n\}$, let $\mathcal{F}_t$ denote the $\sigma$-algebra generated by $X_1, \ldots, X_t$. The idea is to control how much the future (i.e., $X_{t+1}, \ldots, X_n$) can deviate from independence given the past $\mathcal{F}_t$. For each time step $t \in \{1,\ldots,n-1\}$ and any 1-Lipschitz function $g : E_{t+1} \times \cdots \times E_n \to \mathbb{R}$ a deviation measure is defined as
\begin{align}
 \label{Gamma measure}
\Gamma_t(g) := \left\| \mathbb{E}\left[g(X_{t+1}, \ldots, X_n) \mid \mathcal{F}_t\right] - \mathbb{E}\left[g(X_{t+1}, \ldots, X_n)\right] \right\|_\infty,
\end{align}
and a sequence $(X_1, \ldots, X_n)$ is said to satisfy a $\gamma$-mixing condition if there exists a family of non-negative coefficients $(\gamma_{t, m})_{1 \leq t < m \leq n}$ such that, for all such $g$, we have:
\begin{align}
    \label{assumption Gamma}
\Gamma_t(g) \leq \sum_{m = t+1}^n \gamma_{t, m}.
\end{align}
Under Assumption~\eqref{assumption Gamma}, \citet{rio2000} provides a concentration inequality for 1-Lipschitz functions of dependent sequences. Specifically, if $(X_1, \ldots, X_n)$ satisfies the $\gamma$-mixing condition described above and the underlying metric spaces have diameter $\Delta_t$, then the following holds for any positive $s$:
\begin{align}\label{concentration from rio}
\mathbb{E}\left[\exp\left(s f(X_1, \ldots, X_n)\right)\right]
\leq \exp\left\{s \mathbb{E}\left[f(X_1, \ldots, X_n)\right] 
+ \frac{s^2}{8} \sum_{t=1}^n \left( \Delta_t + 2 \sum_{m > t} \gamma_{t, m} \right)^2\right\}.
\end{align}
This result generalizes classical Hoeffding-type inequalities to the dependent setting and is a foundation for many modern generalization bounds involving weakly dependent data. 

Returning to the label prediction again, suppose the variables are object$\times$label pairs. So we are considering $X_i=(U_i,Y_i)$ with the same definitions of labels $Y_i$, predictors $f_\theta$, risk functions $r(\theta)$, $R(\theta)$, and prior distribution $\mu$ and same conditions as we had in the beginning of the paper, with the only difference that here $U_i$ are drawn from arbitrary distributions. Recall that $c$ is the uniform upper bound on the loss function $\ell$, which also means that $\ell$ is $c-$Lipschitz. With this in mind, we have freedom to set our own trivial metrics as follows: for all $t=1,\ldots,n$ we define $d_t\big((u,y),(u',y')\big):=c\mathbf{1}_{(u,y)\neq(u',y')}$. This means that for all $t=1,\ldots,n$ diameter $\Delta_t = c$, and as a consequence function $n\cdot r(\theta)$ forces to be 1-Lipschitz in the space $E_{1}\times\ldots\times E_n$.
\begin{align*}\big|r_{\theta}\left((U_1,Y_1),(U_2,Y_2),\ldots,(U_n,Y_n)\right)-r_{\theta}\left((U'_1,Y'_1),(U'_2,Y'_2),\ldots,(U'_n,Y'_n)\right)\big|\leq\frac{1}{n}(\Delta_1+\ldots+\Delta_n)=c.
\end{align*}
Moreover 1-Lipschitzianity of $n\cdot r(\theta)$ holds for any subspace of the form 
$E_{T+1}\times\ldots\times E_n$, since
\begin{multline*}
\big|r_{\theta}\left((U_1,Y_1),(U_2,Y_2),\ldots,(U_n,Y_n)\right)-r_{\theta}\left((U_1,Y_1),\ldots,(U'_{T+1},Y'_{T+1}),\ldots,(U'_n,Y'_n)\right)\big|
\\
\leq\frac{1}{n}(\Delta_{T+1}+\ldots+\Delta_n)=\frac{n-T}{n}c.
\end{multline*}


Let us denote $C^2=\frac{1}{n}\sum_{t=1}^{n}\big(\Delta_{t} + 2\sum_{m > t} \gamma_{t, m}\big)^2$, and the inequality \eqref{concentration from rio} will take a familiar form: 
\begin{align*}\mathbb{E}\left[\exp \left(s R(\theta)-r(\theta) \right)\right] \leq  e^{\frac{s^2}{8} nC^2}.
\end{align*}
Thus, by following analogous steps of the proof of Theorem~\ref{bern1}.

\begin{thm}\label{time series pac-bayes} Let $U_1,U_2,\ldots,U_n$ be random variables, then for any constants $\lambda>0$, $\delta\in(0,1)$, and prior $\mu\in\mathcal{P}(\Theta)$,   
\begin{align*}
\mathbb{P}_{\mathcal{S}}\Bigg(\exists\rho\in\mathcal{P}(\Theta),\:\: \mathbb{E}_{\theta\sim\rho}\left[R(\theta)\right]\leq\mathbb{E}_{\theta\sim\rho}\left[r(\theta)\right]+\frac{\lambda C^2}{8n}+ \frac{KL(\rho||\mu)+\log \frac{1}{\delta}}{\lambda}  \Bigg) \geq 1-\delta
\end{align*}
where $C^2=\frac{1}{n}\sum_{t=1}^{n}\big(\Delta_{t} + 2\sum_{m > t} \gamma_{t, m}\big)^2$. 
\end{thm}
Note that this result is essentially Theorem 2 of~\citet{PierreAlquier2013}.

\subsection{Second PAC-Bayes Bound for Markov Chains and $\varphi-$mixing} 

As an application to Theorem \eqref{time series pac-bayes}, we investigate the case of stationary, ergodic, $d$-state  Markov chains, with pseudo-spectral gap $\gamma_{ps}>0$. In this case, we are able to get an upper bound on $C^2(\theta)$, and obtain a PAC-Bayes bound that depends on the $\gamma_{ps}$, and thus that can be made empirical.

Before we get to that, let us start with the definition of $\varphi-$mixing coefficients.
\begin{defi}\label{def phi-mixing}
Suppose $(\Omega, \mathcal{E}, \mathbb{P})$ is a probability space, $\mathcal{A} \subset \mathcal{E}$ and $\mathcal{B} \subset \mathcal{E}$ are $\sigma$-fields, then 
$$
\varphi(\mathcal{A}, \mathcal{B}):=\sup _{\substack{A \in \mathcal{A} \\ B \in \mathcal{B}}}\left|\mathbb{P}(B \mid A)-\mathbb{P}(B)\right|.
$$
\end{defi}
We refer the reader to~\citet{doukhan1995mixing} for more details on this definition.
In the time series setting $(U_1,U_2,U_3,...)$ coefficients $\varphi_k$ are defined as
\begin{align}\label{phi defi}
\varphi(k)=\sup_{t\in\mathbb{N}}\varphi\big( \sigma(U_1,\ldots,U_t),\sigma(U_{t+k},U_{t+k+1},\ldots) \big).
\end{align}

Following \citep{rio2000}, the deviation measure $\gamma_{t,m}$ can be characterized by $\varphi-$mixing coefficients as follows
\begin{align*}
 \gamma_{t, m} \leq \Delta_t \varphi(m-t).
 \end{align*}
Hence, we also arrive to the  $\varphi-$mixing version of a PAC-Bayes bound.
 \begin{cor}\label{phi pac-bayes} Let $U_1,U_2,\ldots,U_n$ be random variables, and for each $t = 1, \ldots, n$, let $\varphi(t)$ be defined as in~\eqref{phi defi}, then for any constants $\lambda>0$, $\delta\in(0,1)$, and prior $\mu\in\mathcal{P}(\Theta)$,   
\begin{align*}
\mathbb{P}_{\mathcal{S}}\Bigg(\exists\rho\in\mathcal{P}(\Theta),\:\: \mathbb{E}_{\theta\sim\rho}\left[R(\theta)\right]\leq\mathbb{E}_{\theta\sim\rho}\left[r(\theta)\right]+\frac{\lambda \Phi}{8n}+ \frac{KL(\rho||\mu)+\log \frac{1}{\delta}}{\lambda}  \Bigg) \geq 1-\delta
\end{align*}
where $\Phi=\frac{1}{n}\sum_{t=1}^{n}\big(\Delta_{t} + 2\left( \Delta_t\varphi(1) + \Delta_t\varphi(2) + \ldots + \Delta_t\varphi(n)   
\right)\big)^2$. 
\end{cor}
Now let us make the connection of $\varphi$-mixing, $t_{\text{mix}}$, and $\gamma_{ps}$\label{connection}.
Given a stationary, ergodic Markov chain \( U \), the coefficients \( \varphi(k) \) are known to be expressed in terms of the distance to equilibrium:$$\varphi(k):=\varphi_U(k)=\sup_{u\in\mathcal{U}}\|P^k(u,\cdot)-\pi(\cdot)\|_{TV},$$
as proven by~\citet{davydov1968convergence}.
This measure is non-increasing and is endowed with a sub-multiplicative property, namely
\begin{alignat*}{2}
\varphi(t_1)&\leq\varphi(t_2)\quad\quad\quad\quad&&[\text{when}\:\:t_1\leq t_2]\\
\varphi(t_1+t_2)&\leq2\varphi(t_1)\varphi(t_2)&&[\text{for any}\:\:t_1,t_2]
\end{alignat*}
Denoting $a=\ceil*{\frac{k}{t_{\text{mix}}}}$, and $\rho=\left(\frac{1}{2}
\right)^{\frac{1}{t_{\text{mix}}}}$, then applying sub-multiplicative properties we derive
$$\varphi(k)\leq\varphi\left(a \cdot t_{\text{mix}}\right) \leq  
2^{a-1}\varphi(t_{\text{mix}})^{a} \leq 2^{a-1}\left(\frac{1}{4}\right)^{a} = \left(\frac{1}{2}\right)^{a+1}\leq\frac{1}{2}\cdot\left(\frac{1}{2}
\right)^{\frac{k}{t_{\text{mix}}}}=\frac{1}{2}\cdot\rho^{k} .$$
\\
On the other hand using an upper bound on $t_{\text{mix}}$ \citet{paulin2015}, we have 
\[
b(\pi_*)\, \gamma_{ps} \;\leq\; \frac{1}{t_{\text{mix}}}
\]
with $b(\pi_*) = \left(\ln  \frac{1}{\pi_*}  + 2 \ln 2 + 1\right)^{-1}$. Subsequently, for ergodic Markov chains with pseudo-spectral gap $\gamma_{ps}$ we are able to derive bounds on $\rho$ and $\varphi(k)$ which depend on $\gamma_{ps}$. 
\begin{align*}
    \rho =\left(\frac{1}{2}\right)^\frac{1}{t_{\text{mix}}} \leq \left(\frac{1}{2}\right)^{b(\pi_*)\gamma_{ps}}\text{and} \quad\varphi(k) \leq \frac{1}{2}\cdot \rho^k \leq \left( \frac{1}{2} \right)^{k \cdot b(\pi_*) \gamma_{ps}+1}
\end{align*}

with $\alpha=\frac{C_{\text{ps}}d}{\varepsilon\sqrt{\pi_*}} e ^ { -n\varepsilon^2\pi_*\min\{\gamma_{ps},\frac{1}{C(P)}\} } $.

Thus, applying the aforementioned bounds, we derive
\begin{align}\label{eq: c-square bound}
 \Phi&=
 \frac{1}{n}\sum_{t=1}^{n}\bigg(\Delta_{t} + 2\left( \Delta_t\varphi(1) + \Delta_t\varphi(2) + \ldots + \Delta_t\varphi(n)   
\right)\bigg)^2\notag\\
    & \leq \frac{c^2}{n}\sum_{t=1}^{n}\bigg(1 + 2\left( \frac{1}{2}\rho + \frac{1}{2}\rho^2 + \ldots +\frac{1}{2}\rho^n  
\right)\bigg)^2\notag\\
    & \leq c^2\cdot \bigg(\frac{1-\rho^{n+1}}{1-\rho}\bigg)^2\notag\\&\leq c^2\cdot \left(\frac{1-\left(\frac{1}{2}\right)^{(n+1)b(\pi_*)\gamma_{ps}}}{1-\left(\frac{1}{2}\right)^{b(\pi_*)\gamma_{ps}}}\right)^2
\end{align}
where $b(\pi_*)=(\ln{\frac{1}{\pi_*}}+2\ln{2}+1)^{-1}$.

This brings us to the following thoerem. 
\begin{thm}
    Assume $\{U_t\}_{t =1}^n$ be a stationary, ergodic, finite state Markov chain with pseudo-spectral gap $\gamma_{ps}>0$, then for any constants $\lambda>0$, $\delta\in(0,1)$, and prior $\mu\in\mathcal{P}(\Theta)$,   
\begin{align*}
\mathbb{P}_{\mathcal{S}}\left(\forall\rho\in\mathcal{P}(\Theta),\:\: \mathbb{E}_{\theta\sim\rho}\left[R(\theta)\right]\leq\mathbb{E}_{\theta\sim\rho}\left[r(\theta)\right]+   \frac{\lambda c^2 }{8n}  \left(\frac{1-\left(\frac{1}{2}\right)^{(n+1)b(\pi_*)\gamma_{ps}}}{1-\left(\frac{1}{2}\right)^{b(\pi_*)\gamma_{ps}}}\right)^2  + \frac{KL(\rho||\mu)+\log \frac{1}{\delta}}{\lambda}  \right)\geq 1-\delta
\end{align*}
where $b(\pi_*)=(\ln{\frac{1}{\pi_*}}+2\ln{2}+1)^{-1}$.
\end{thm}
As before, one can substitute \( \gamma_{ps} \) with an empirical estimate \( \widehat{\gamma}_{ps} \) to obtain an empirical bound.

It would be very nice to get an empirical version of the PAC-Bayes bounds with the $\varphi$ coefficients without assuming the Markov property. However, this would require to estimate the $\varphi$-mixing coefficients, which is still an open question.

\end{document}